\documentclass[twoside]{article}

\usepackage[dvipsnames]{xcolor}
\usepackage[
    colorlinks=true,
    citecolor=red,
    linkcolor=blue,
    urlcolor=blue
]{hyperref}

\usepackage{amsmath,amsfonts,amsthm,amssymb}
\usepackage{mathtools}
\usepackage{enumitem}
\providecommand{\texorpdfstring}[2]{#1}

\newtheorem{theorem}{Theorem}[section]
\newtheorem{lemma}[theorem]{Lemma}
\newtheorem{proposition}[theorem]{Proposition}
\newtheorem{corollary}[theorem]{Corollary}

\theoremstyle{definition}
\newtheorem{definition}[theorem]{Definition}
\newtheorem{assumption}[theorem]{Assumption}
\newtheorem{remark}[theorem]{Remark}
\newtheorem{example}[theorem]{Example}

\newcommand{\X}{\mathcal{X}}
\newcommand{\Pcal}{\mathcal{P}}

\newcommand{\pack}{\mathrm{pack}}

\newcommand{\EE}{\mathbb{E}}
\newcommand{\OPT}{\mathrm{OPT}}
\newcommand{\cont}{\mathrm{cont}}

\DeclareMathOperator{\UCB}{UCB}
\DeclareMathOperator{\LCB}{LCB}

\newcommand{\Sact}{{\mathcal{S}_{\mathrm{act}}}}

\newcommand{\NumProbes}{N_{\text{probe}}}

\usepackage{lmodern}

\usepackage[accepted]{aistats2026}

\usepackage[round]{natbib}

\begin{document}

\runningauthor{Chakraborty* and Rege* et al.}
\newcommand{\CU}{$^{1}$}
\newcommand{\INRIA}{$^{2}$}
\newcommand{\EQUAL}{$^{*}$}

\twocolumn[

\aistatstitle{Multi-Agent Lipschitz Bandits}

\aistatsauthor{Sourav Chakraborty\CU\EQUAL, Amit Kiran Rege\CU\EQUAL, Claire Monteleoni\CU\INRIA, Lijun Chen\CU}

\aistatsaddress{\CU University of Colorado Boulder \quad \INRIA INRIA Paris \\ \EQUAL Equal contribution}
]
\makebox[0pt][l]{\begingroup\fontfamily{cmr}\fontsize{1}{1}\selectfont\textcolor{white}{x}\endgroup}

\begin{abstract}
We study the decentralized multi-player stochastic bandit problem over a continuous, Lipschitz-structured action space where hard collisions yield zero reward. Our objective is to design a communication-free policy that maximizes collective reward, while separating coordination costs from learning costs. We propose a modular protocol that first solves the multi-agent coordination problem by identifying and seating players on distinct, high-value regions via a novel maxima-directed search and then decouples the problem into $N$ independent single-player Lipschitz bandits. In the consensus regime, we obtain an end-to-end regret bound whose dominant learning term is \(\tilde{O}(T^{(d+1)/(d+2)})\), matching the single-player Lipschitz rate; the upfront coordination cost is horizon-independent at fixed confidence and only polylogarithmic in \(T\) in the expected-regret form. Under an additional public coverage/scheduling assumption for the epochic extension, we also obtain a gap-free \(\tilde{O}(T^{(d+1)/(d+2)})\) guarantee. We further derive a matching lower bound for the dominant learning term and extend the framework to general distance-threshold collision models.
\end{abstract}

\section{INTRODUCTION} \label{sec:intro}

Many sequential decision-making problems involve multiple autonomous agents operating in a shared environment without a central controller \citep{boursier2020survey, landgren2020distributed}. Consider a team of cognitive radios \citep{jouini2012decision,liu2010distributed,anandkumar2011distributed} searching for unoccupied, high-quality frequency bands, or a fleet of drones coordinating to survey distinct, high-value areas. In these scenarios, agents must learn the value of different actions from stochastic feedback, a classic exploration-exploitation dilemma \citep{lai1985asymptotically,auer2002finite,slivkins2019introduction,lattimore2020bandit}. However, three fundamental challenges arise: the action space is often continuous \citep{Kleinberg,bubeck2011xarmed,magureanu2014lipschitz}, agents may interfere with each other through hard collisions \citep{rosenski2016multi}, and they must act without direct communication.

This paper addresses the confluence of these three challenges within the framework of multi-player stochastic bandits. We consider a cooperative setting where $N$ players share a continuous action domain. The environment is partitioned into a finite set of regions. If two or more players choose actions in the same region at the same time, a ``hard collision" occurs, and all colliding players receive zero reward and no information. This models contention for a rivalrous resource. The reward structure of the continuous domain is governed by an unknown but smooth function, which we model as being Lipschitz continuous. The goal for the collective is to maximize the total reward over a time horizon $T$.

Most prior work on multi-player bandits has focused on settings with discrete, finite action sets \citep{agarwal2025multiplayer,rosenski2016multi,wang2020optimal}. While foundational, these models do not capture applications where actions are inherently continuous, such as setting a price, tuning a physical parameter, or choosing a location. The introduction of a continuous action space, combined with decentralization and collisions, presents a formidable challenge. A naive discretization of the action space would be computationally intractable and statistically inefficient. The Lipschitz structure \citep{Kleinberg,magureanu2014lipschitz,chakraborty2025incentivized} is key, as it allows for generalization: the reward at one point provides information about rewards at nearby points. However, this structure also introduces a new subtlety: the value at the center of a region can be a poor proxy for the maximum value achievable within it, especially if the reward function has sharp peaks near region boundaries (see Figure \ref{fig:pathology}). An effective strategy must therefore be sensitive to the maxima of regions, not just their centers.

Our work provides a principled, end-to-end solution for this problem. We propose a fully decentralized, multi-phase algorithm that decouples the problem of coordination from learning. The core idea is to spend a coordination budget at the start to solve the multi-agent allocation problem first: identify a high-value set of $N$ distinct regions and assign one player to each. For fixed confidence, this coordination budget is horizon-independent, and in the expected-regret guarantee it contributes only polylogarithmic dependence on $T$ through the failure budget. Once this seating is achieved, the problem factorizes into $N$ independent single-player Lipschitz bandit problems, which can be solved with near-optimal efficiency for the remaining duration.

\begin{figure}[t]
  \centering
  \includegraphics[width=0.8\columnwidth]{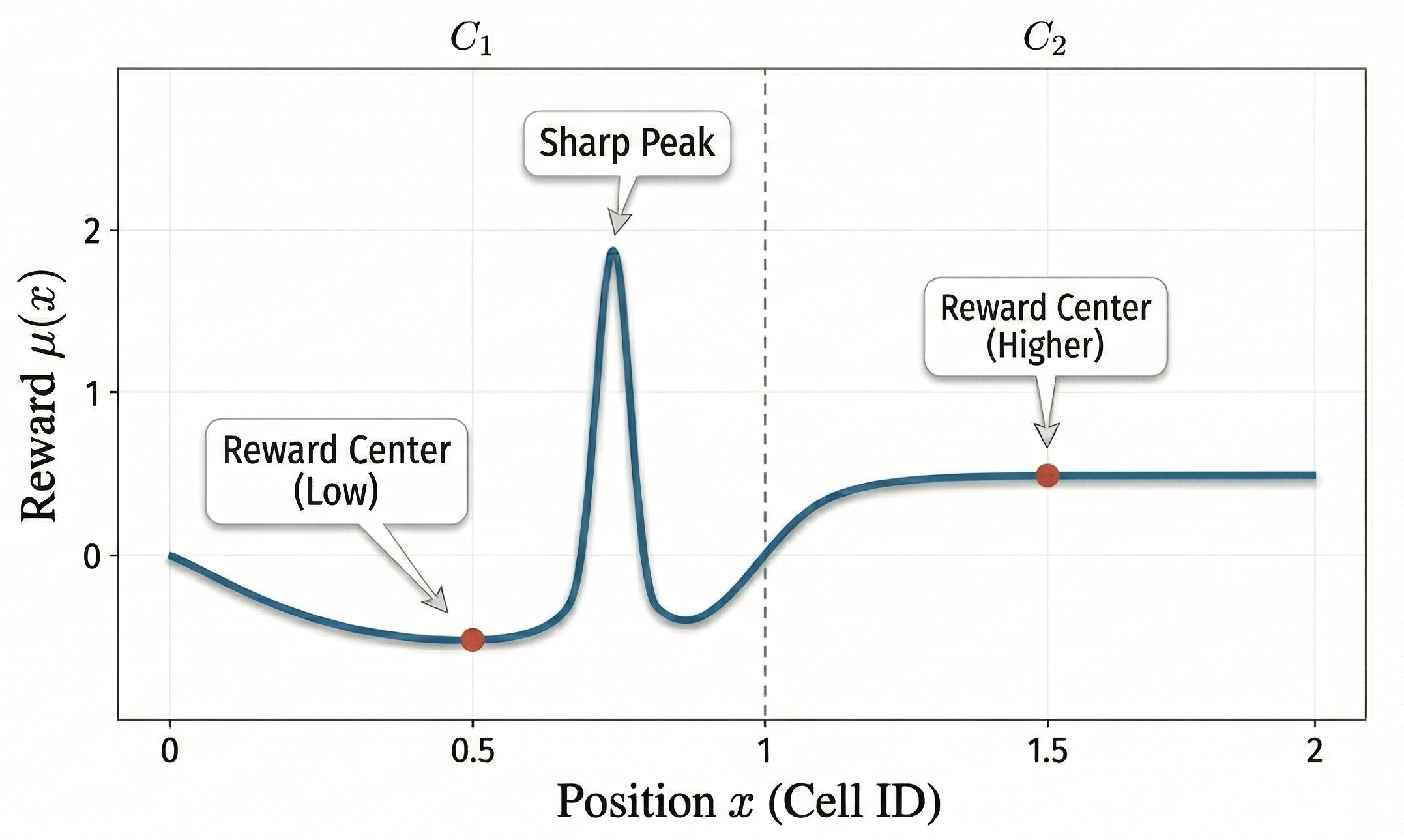}
  \caption{\textbf{The center-vs-maximum pathology in 1D}. In cell $C_1$, the reward function $\mu(x)$ has a modest center value $\mu(x_{C_1})$ but contains a sharp peak near its boundary, making its true maximum $\mu^*(C_1)$ optimal.}
  \label{fig:pathology}
\end{figure}

Our contributions establish a first end-to-end solution for this problem. First, we introduce a modular, communication-free protocol that decouples coordination from learning. Its core innovation is a maxima-directed identification phase that performs a ``local peek'' inside candidate regions to bracket true cell suprema, provably avoiding the biases of simpler center-based rankings. Second, we analyze the practical decentralized Musical Chairs routine, where players do not know which target cells are already occupied, and show that it seats all \(N\) players in \(O(N)\) expected time. Third, under a top-\(N\) separation condition ensuring consensus, we obtain an end-to-end regret bound whose dominant learning term matches the optimal single-player Lipschitz rate, namely \(\tilde{O}(T^{(d+1)/(d+2)})\). We show that for fixed confidence, the coordination term is horizon-independent, while in the expected-regret form it contributes only polylogarithmic dependence on $T$ through the failure budget. Fourth, in the \emph{gap-free} setting, we show that the same single-player rate can be recovered under a public coverage/scheduling assumption and a near-optimality-dimension condition. Finally, we prove a matching lower bound for the dominant learning term, showing that the \(T^{(d+1)/(d+2)}\) dependence cannot be improved in the regimes covered by our upper bounds, and we extend the framework to general distance-threshold collision models. To our knowledge, this is the first work to provide such guarantees for multi-player bandits in continuous domains.

\section{PRELIMINARIES} \label{sec:prelims}

We build upon three established areas in sequential decision-making: the stochastic multi-armed bandit problem, its extension to continuous arms via the Lipschitz assumption, and the multi-player variant with collisions. We briefly review each to establish notation and context.

\subsection{Stochastic Multi-Armed Bandits}
The canonical multi-armed bandit (MAB) problem involves a single player sequentially choosing from a set of $K$ discrete actions, or ``arms". At each time step $t$, the player selects an arm $i_t \in \{1, \dots, K\}$ and receives a stochastic reward drawn from an unknown distribution with mean $\mu_i$. The player's goal is to maximize the cumulative expected reward over a horizon $T$. Performance is measured by the cumulative regret, defined as the expected difference between the reward from always playing the single best arm and the reward accumulated by the player's policy:
\[
R(T) = T \cdot \mu^* - \sum_{t=1}^T \EE[\mu_{i_t}],
\]
where $\mu^* = \max_{i \in \{1,\dots,K\}} \mu_i$. Minimizing regret requires balancing the exploration of arms to learn their mean rewards with the exploitation of the arm that currently seems best.

\subsection{Lipschitz Bandits in Continuous Domains}
When the set of actions is a continuous domain, such as $\X = [0,1]^d$, the problem becomes intractable without further assumptions, as there are infinitely many arms to explore. The Lipschitz bandit model introduces structural smoothness. The unknown mean-reward function $\mu: \X \to [0,1]$ is assumed to be $L$-Lipschitz with respect to a norm, typically the Euclidean norm $\|\cdot\|_2$:
$$
|\mu(x) - \mu(y)| \le L \|x - y\|_2 \quad \text{for all } x, y \in \X.
$$
This condition ensures that the mean rewards of nearby points are similar, allowing an algorithm to generalize from a finite number of samples to the entire space. This structure makes the problem tractable, and algorithms for this setting can achieve near-optimal regret that scales as $\tilde{O}(T^{(d+1)/(d+2)})$, where the exponent depends on the dimension $d$ of the action space.

\subsection{Cooperative Multi-Player Bandits and Collisions}
In the multi-player MAB setting, $N$ players simultaneously choose from a common set of arms. We focus on the cooperative goal, where the objective is to maximize the sum of rewards across all players. A central challenge is handling collisions. In the "hard collision" model, if two or more players select the same arm (or region) in the same round, all colliding players receive zero reward. This creates an incentive for players to coordinate on distinct, high-value arms.

A simple and effective communication-free protocol for this coordination task is known as Musical Chairs. Once a set of $N$ high-quality arms has been identified, the players must assign themselves to these arms without conflict. In the Musical Chairs protocol, each unassigned player repeatedly samples an arm uniformly from the target set. If a player lands on an arm that no one else chose in that round, they "seat" there and play that arm for the remainder of the game. This process continues until all players are seated. A crucial aspect of our analysis considers the practical implementation where players do not know which of the $N$ target arms are already occupied, making their sampling choices over the full set of $N$ target arms.

\section{PROBLEM SETUP AND BENCHMARKS}
\label{sec:setup}

We consider a decentralized, cooperative stochastic bandit problem with $N$ players acting over a shared, continuous action domain. Our goal is to design a communication-free policy that allows players to achieve near-optimal collective reward, where the costs of coordination are provably independent of the time horizon $T$ for a fixed confidence.

\subsection{Actions, Rewards, and Lipschitz Structure}

The action space for each of the $N$ players is the compact set $\X = [0,1]^d$, endowed with the Euclidean norm $\|\cdot\|_2$. At each round $t=1,2,\dots,T$, every player $j\in[N]$ chooses an action $X_t^{(j)}\in \X$. The rewards are governed by an unknown mean-reward function $\mu:\X\to[0,1]$ that is assumed to be $L$-Lipschitz continuous:
\[
|\mu(x)-\mu(y)| \le L \|x-y\|_2 \qquad \text{for all }x,y\in\X .
\]

We assume that the players know a common upper bound on the Lipschitz constant; for notational simplicity, we denote this available upper bound by $L$. Likewise, the number of players $N$, the partition $\Pcal$, and the synchronous round structure are common knowledge. Note that the assumption of a known Lipschitz constant (or a known upper bound) is standard in much of the Lipschitz bandit literature~\citep{magureanu2014lipschitz,bubeck-xarmed,Kleinberg}.

When a player's action does not result in a collision, they observe a stochastic reward drawn from a distribution with mean $\mu(X_t^{(j)})$ and independent $1$-sub-Gaussian noise. The Lipschitz property provides the essential smoothness structure that makes learning over a continuous space feasible.

\subsection{A Tractable Collision Model for Continuous Spaces}

Defining a meaningful collision model is a primary conceptual hurdle in continuous domains. If a collision were defined as two players selecting the exact same point, such an event would occur with zero probability, rendering the notion trivial. A practical model must instead capture the idea that players interfere when they operate in "proximate" regions of the action space.

As one of the first approaches for this new setting, we introduce a collision geometry based on a fixed, discretized partition of the space (see Figure \ref{fig:geometry}(a)). We assume the action space $\X$ is partitioned into a set $\Pcal=\{C_1,\dots,C_K\}$ of $K=\lceil 1/h\rceil^d$ disjoint hypercubic cells, each of side-length $h$. We assume there are enough cells to accommodate all players, $K\ge N$. A \emph{hard collision} occurs for a set of players if, at the same round, they all choose actions within the \emph{same} cell $C \in \Pcal$. When a collision occurs in a cell, every player involved receives a null observation, denoted $\bot$, and a reward of zero.

This partition-based model is a natural and analyzable abstraction for many real-world systems where operational zones are discrete by design. For example, in cognitive radio networks, the spectrum is divided into discrete channels; in logistics, a city is divided into service zones for delivery drones; and in cloud computing, resources may be allocated from discrete server clusters. In these cases, interference is determined by co-location within a predefined region, not just by continuous proximity.

\begin{figure}[t]
  \centering
  \includegraphics[width=0.45\textwidth]{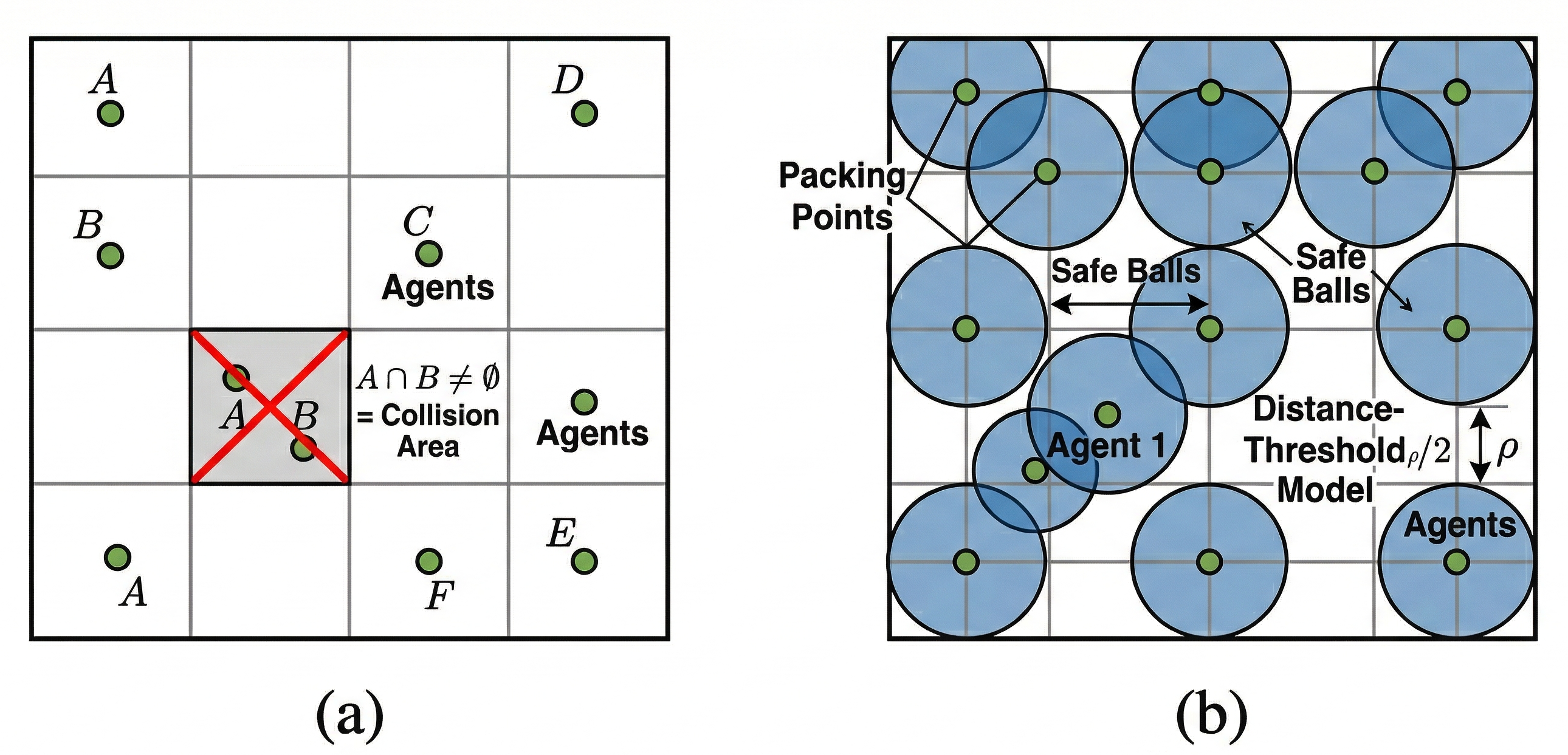}
  \caption{\textbf{Collision Geometries}. (a) The partition model discretizes the action space $\mathcal{X}$ into $K$ disjoint hypercubic cells; collisions occur when multiple players occupy the same cell. (b) The distance-threshold model manages interference by seating players within safe balls $B_i$ of radius $\sigma$ around $r$-packing centers $z_i$, ensuring an inter-agent separation $>\rho$.}
  \label{fig:geometry}
\end{figure}

While this partition model provides a clean and practical foundation, our algorithmic framework is sufficiently general to accommodate other geometries. In later sections, we will show how our approach extends to a distance-threshold model, where a collision occurs if any two players' actions are within a certain Euclidean distance $\rho$ of each other (see Figure \ref{fig:geometry}(b)). That model is more suited to applications like mobile robotics or sensor networks where interference is governed by physical proximity. By first solving the partition-based problem, we establish the core algorithmic principles in a clear and simple setting.

\subsection{Performance Benchmark and Objective}

The collision model imposes a fundamental constraint: at most one player can earn a non-zero reward from any given cell $C$ in a single round. It is therefore inappropriate to compare the system's performance to the ideal single-player benchmark of $N \cdot \sup_{x \in \X} \mu(x)$, as this might require all $N$ players to occupy the same infinitesimally small region. A principled comparator must respect this feasibility constraint.

We therefore define the benchmark based on the best possible static assignment of players to $N$ distinct cells. For each cell $C \in \Pcal$, let its optimal value be its cell-wise supremum, $\mu^\ast(C) := \sup_{x\in C}\mu(x)$. Let $\mu^\ast_{(1)}\ge \mu^\ast_{(2)}\ge\dots\ge \mu^\ast_{(K)}$ be these values sorted in nonincreasing order. The optimal collision-feasible reward in a single round is the sum of the top $N$ cell maxima: $\OPT_{\cont}(\Pcal,N) \;:=\; \sum_{m=1}^N \mu^\ast_{(m)}$.

The cumulative regret of a decentralized policy $\pi$ over a horizon $T$ is the difference between this optimal benchmark and the expected total reward collected by all players:

\begin{equation}
\label{eq:cont-regret}
\begin{aligned}
R_{\cont}(T; \pi, \Pcal, N) &:= T \cdot \OPT_{\cont}(\Pcal, N) \\
&\quad - \EE_\pi \! \left[ \sum_{t=1}^T \sum_{j=1}^N r_t^{(j)} \right],
\end{aligned}
\end{equation}
where $r_t^{(j)}$ is the realized reward for player $j$ at round $t$. Note that the benchmark depends on the partition $\Pcal$. This is an intentional feature of the model, not an artifact of the analysis; the partition defines the physical collision geometry of the environment.

A subtlety in this continuous setting is the difference between a cell's center value, $\mu(x_C)$, and its true maximum, $\mu^\ast(C)$. While the Lipschitz property guarantees they are close - $|\mu^\ast(C) - \mu(x_C)| \le Lh\sqrt{d}/2$ - their relative ordering across cells can be completely different. For instance, a cell with a modest center value might contain a sharp peak near its boundary, making it more valuable than a cell with a higher center value that is relatively flat. Any successful algorithm must therefore be designed to identify cells based on their maxima, not just their centers.

We use $[m]$ for the set $\{1,\dots,m\}$. All of our guarantees are high-probability statements that hold simultaneously for all players, cells, and time steps. We manage this by defining a total failure probability budget $\delta_{\textrm{sys}}\in(0,1/4)$ and allocating portions of it, such as $\delta_{\text{I}}$ and $\delta_{\text{II}}$, to different phases of the algorithm. This ensures the entire system behaves as expected with probability at least $1-\delta_{\textrm{sys}}$.

Since this is, to our knowledge, the first work to extend decentralized multi-player bandits to Lipschitz continuous domains with hard collisions, we focus on the cleanest foundational setting: decentralized play with zero-reward collisions and fixed public geometry. Extensions to richer models are left for future work.

\section{OUR APPROACH: A MULTI-PHASE DECENTRALIZED PROTOCOL}

Before going into technical details, we present a high-level overview of our strategy. The core idea is to decouple the multi-agent coordination problem from the single-agent continuous optimization problem. We achieve this with a four-phase protocol where the first three phases are dedicated to coordination and incur a total cost that is \emph{independent} of the time horizon $T$ up to logarithmic factors coming from the failure probability.

\begin{itemize}
    \item \textbf{PHASE I: COARSE IDENTIFICATION.} For a fixed duration $T_0$, all players explore the space by sampling cell centers uniformly at random. This process is intentionally chaotic and communication-free. While many samples will result in collisions, we show that with high probability, every player obtains a sufficient number of successful, non-colliding observations from every cell to construct coarse but statistically valid confidence bounds on each cell's maximum value. The purpose of this phase is not to be precise, but to safely prune the vast majority of suboptimal cells.

    \item \textbf{PHASE II: USING MAXIMA FOR REFINEMENT.} Using the candidate cells identified in Phase I, players perform a localized ``peek" inside each one. They sample from a fine grid of points within each candidate cell to build high-resolution confidence bounds that tightly bracket the true cell maximum $\mu^\ast(C)$. This critical step corrects for the center-versus-maximum bias and allows us to identify the top-$N$ cells, even without a gap between the best and the rest.

    \item \textbf{PHASE II$\tfrac{1}{2}$: DECENTRALIZED SEATING.} Having agreed upon a common set of $N$ target cells, players must assign themselves to these cells without conflict. They use the Musical Chairs protocol, where unseated players repeatedly sample from the target set until they land on a free cell. We analyze the practical version of this protocol and show that all players are seated in expected $O(N)$ time.

    \item \textbf{PHASE III: WITHIN-CELL OPTIMIZATION.} Once each player is uniquely assigned to a high-quality cell, the multi-agent problem factorizes. For the remainder of the horizon, each player independently runs a single-player Lipschitz bandit algorithm confined to their assigned cell, efficiently optimizing their local reward.
\end{itemize}
This modular structure allows us to isolate and solve the challenges of coordination and learning sequentially, leading to a cleaner analysis.

\section{PHASE I: COARSE IDENTIFICATION}
\label{sec:phase1}

The first phase aims to solve a difficult task with a simple tool: uniform random exploration. The goal is for every player to obtain a coarse but reliable confidence bracket on each cell's maximum value, $\mu^\ast(C)$. The exploration is ``collision-censored"-players make no attempt to avoid each other and instead rely on randomness to provide a sufficient number of non-collision events to learn from.

\subsection{Phase I Protocol}

This phase runs for a fixed budget of $T_0$ rounds. In each round $t \in [T_0]$, every player $j$ independently samples a cell $C_t^{(j)}$ uniformly at random from $\Pcal$ and probes its center $x_{C_t^{(j)}}$ and observes a reward $Y_t^{(j)}$ if she is the unique occupant of the cell, or the null symbol $\bot$ otherwise.

Let $U_j(C,t)$ be the indicator that player $j$ uniquely occupies cell $C$ at round $t$. We track the \emph{success count} for each player-cell pair:
\begin{equation}
\label{eq:success-count}
o_{j,C}(T_0) := \sum\nolimits_{t=1}^{T_0} U_j(C,t).
\end{equation}
If $U_j(C,t)=1$, the player observes a reward $Y_t^{(j)}$; otherwise, she receives the null symbol $\bot$ . The per-round success probability $p_K$ remains constant. Since players sample independently, $p_K$ is the probability that player $j$ selects $C$ while all others avoid it:
\begin{equation}
\label{eq:pk-def}
p_K := \tfrac{1}{K}\left(1-\frac{1}{K}\right)^{N-1}.
\end{equation}

From these successful observations, each player computes an empirical mean for each cell's center, $\widehat{\mu}_{j}^{(0)}(C)$. This estimate, combined with a concentration radius $r^{(0)}_{j}(C)$ and the geometric bracket from the Lipschitz property, yields initial lower and upper confidence bounds on the cell's true maximum value, $\mu^\ast(C)$. 

Formally, the initial confidence brackets for each cell maximum $\mu^*(C)$ are:
\begin{equation}
\label{eq:phase1-brackets}
\begin{aligned}
\LCB^{(0)}_j(C) &:= \widehat{\mu}_{j}^{(0)}(C) - r^{(0)}_{j}(C), \\
\UCB^{(0)}_j(C) &:= \widehat{\mu}_{j}^{(0)}(C) + r^{(0)}_{j}(C) + (Lh\sqrt{d})/2,
\end{aligned}
\end{equation}
where $r^{(0)}_{j}(C)$ is the concentration radius defined with $\beta_0 := \log \frac{4NK(T_0+1)}{\delta_I}$:
\begin{equation}
\label{eq:r0}
r^{(0)}_{j}(C) := \sqrt{\frac{\beta_0}{2\,\max\{1, o_{j,C}(T_0)\}}}.
\end{equation}

\subsection{Phase I Guarantees}
Despite the collision-prone nature of the exploration, standard concentration inequalities show that this simple protocol is highly effective. The first result establishes that the number of successes, $o_{j,C}(T_0)$, is sharply concentrated around its mean for all players and cells simultaneously. The second shows that the empirical means are accurate estimates of the true center values.

\begin{lemma}[Success Counts Under Collisions]
\label{lem:counts}
For any $\eta\in(0,1)$, with probability at least $1-\delta_{I}/2$, the success count for every player $j\in[N]$ and every cell $C\in \Pcal$ is bounded by $(1\pm\eta) T_0 p_K$.
\end{lemma}

\begin{lemma}[Anytime Concentration for Center Means]
\label{lem:means}
With probability at least $1-\delta_{I}/2$, for every player $j\in[N]$ and every cell $C\in \Pcal$, the empirical mean is close to the true mean: $|\widehat{\mu}_j^{(0)}(C) - \mu(x_C)| \le r^{(0)}_{j}(C)$. This holds for any realized value of the success count $o_{j,C}(T_0)$.
\end{lemma}

The proofs of these lemmas, which involve standard applications of Chernoff and Hoeffding bounds with a union bound over all players and cells, are deferred to the appendix. By combining these statistical guarantees with the geometric bound derived from the Lipschitz property, we arrive at the main result of Phase I: with high probability, every player constructs a valid confidence interval for every cell's true maximum value.

\begin{proposition}[Phase-I Maxima Brackets]
\label{prop:phase1-brackets}
With probability at least $1-\delta_{I}$, for every player $j\in[N]$ and every cell $C\in \Pcal$, the computed bounds are valid: $\LCB^{(0)}_j(C) \le \mu^\ast(C) \le \UCB^{(0)}_j(C)$.
\end{proposition}

\section{PHASE II: ZOOMING IN ON CELLS}
\label{sec:phase2}

Phase I provides each player with valid but wide confidence brackets on each cell's potential. This initial map is crucial for pruning the search space, but it is not sharp enough for final decision-making, primarily due to the center-versus-maximum bias. 

The purpose of Phase II is to resolve this ambiguity. Here, players ``zoom in" on the most promising regions identified in Phase I, conducting a localized exploration - which we call a ``local peek" - to refine their estimates and construct confidence bounds on the true cell maxima, $\mu^{\ast}(C)$, down to a pre-specified target accuracy, $\varepsilon > 0$.

The phase begins with each player independently forming a smaller ``active set" of candidate cells, $\Sact^{(0)}_j$. This is a safe elimination step: using their Phase I brackets, players discard any cell whose most optimistic outcome (its upper bound) cannot compete with the most pessimistic outcome (the lower bound) of the top-$N$ candidate cells. Players may form different active sets due to the randomness in their Phase I observations; our analysis handles this by considering the maximum active set size, $M_{act}:=\max_j |\Sact^{(0)}_j|$.

Within each of their active cells, players then conduct the local peek. For a fixed duration $T_1$, they sample from a fine $\eta$-net of probe points laid out inside each cell. An $\eta$-net is a grid of points so dense that any point in the cell is within a distance $\eta$ of some grid point. This strategy allows us to approximate the supremum of the Lipschitz function over the continuous cell by taking the maximum over a finite set of points. A probe at one of these points is successful only if no other player samples any point within the same cell in that round. While still subject to collisions, this exploration is highly focused on the regions that matter most.

\subsection{Collision-Tolerant Probe Sampling}
A key technical challenge is to ensure that, despite collisions and decentralization, every player gathers enough information at \emph{every} probe point. Our analysis shows that a carefully chosen phase duration $T_1$ is sufficient to guarantee this with high probability.

\begin{lemma}[Phase-II Probe Coverage]\label{lem:coverage}
Let $q_{M_{act},\eta}$ be the per-round success probability for a (player, cell, probe) triple, and let $\NumProbes$ be the total number of probe points across all players' active sets. Choose integers
\[
b \ge 4\log(2\NumProbes/\delta_{II}) \qquad\text{and}\qquad T_1 \ge 2b/q_{M_{act},\eta}.
\]
Then, with probability at least $1-\delta_{II}/2$, every triple attains at least $b$ non-collision samples.
\end{lemma}

With a guaranteed budget of at least $b$ successful samples per probe point, standard concentration inequalities ensure that the empirical mean at each point is a highly accurate estimate of its true mean. This accuracy at the probe points translates directly into a tight bound on the cell's maximum value.

\begin{proposition}[Refined Maxima Brackets]
\label{prop:refined-bracket}
At the end of Phase II, by choosing parameters $\eta$ and $b$ appropriately for a target accuracy $\varepsilon$, each player $j$ constructs a new, refined bracket $[\LCB^{(1)}_j(C), \UCB^{(1)}_j(C)]$ for each of her active cells. With high probability, this bracket contains the true maximum $\mu^\ast(C)$ and has a width of at most $\varepsilon$.
\end{proposition}

\subsection{Selecting the Top-$N$ Cells}
With these tight and reliable brackets on the true cell maxima, players are equipped to make their final selection. We adopt a deterministic rule to select $N$ cells: each player selects the $N$ cells corresponding to their largest lower confidence bounds, $\LCB^{(1)}_j(C)$, breaking any ties with a fixed, public ordering of the cells (e.g. say lexicographic). This guarantees every player outputs a set, $S^{(j)}_\varepsilon$, of size exactly $N$.

\begin{theorem}[Gap-Free $\varepsilon$-Optimality]\label{thm:gapfree-safety}
With high probability, the set $S^{(j)}_\varepsilon$ chosen by any player $j$ is $\varepsilon$-optimal: every cell $C \in S^{(j)}_\varepsilon$ satisfies $\mu^\ast(C) \ge \mu^\ast_{(N)} - \varepsilon$.
\end{theorem}

This powerful guarantee holds for any reward function, irrespective of the gaps between cell values. It ensures our procedure is robust even when the decision is difficult. Remarkably, this simple, decentralized rule leads to a powerful emergent behavior: under a mild separation condition on the reward function, it forces all players to agree on the exact same set of top-$N$ cells, achieving consensus without any communication.

\begin{definition}[$\varepsilon$-Uniqueness at the Top-$N$]
\label{def:eps-uniq}
A mean function $\mu$ is $\varepsilon$-unique at the top-$N$ if there is a unique set $S^\dagger \subset \Pcal$ of size $N$ such that $\min_{C\in S^\dagger} \mu^\ast(C) \ge \max_{C\notin S^\dagger} \mu^\ast(C) + 2\varepsilon$.
\end{definition}

\begin{lemma}[Consensus under $\varepsilon$-Uniqueness]\label{lem:consensus}
If the instance is $\varepsilon$-unique and the bracket width is at most $\varepsilon$, then on our high-probability event, all players select the exact same set of cells: $S^{(j)}_\varepsilon = S^\dagger$ for all $j \in [N]$.
\end{lemma}

This decentralized agreement is a cornerstone of our protocol's success since it allows the final seating phase to happen.

\begin{example}[Center-vs-maximum pathology in 1D]\label{ex:center-failure}
Let $d=1$, $h=\tfrac12$, so $\Pcal=\{C_1=[0,\tfrac12],\,C_2=[\tfrac12,1]\}$.
Define a Lipschitz mean $\mu(x)=x+\alpha \phi(x)$ with a narrow bump $\phi$ of height $1$ supported in $[\tfrac12-\delta,\tfrac12]$, with $\delta\ll h$ and $\alpha>0$ small so that $L$ is finite.
Then $\mu(x_{C_1})<\mu(x_{C_2})$ (center ranking favors $C_2$) but $\mu^\ast(C_1)>\mu^\ast(C_2)$ (the bump near the boundary makes $C_1$ optimal).
Phase~I center estimates therefore mis-rank the cells, while Phase~II's local peek in $C_1$ identifies the bump and restores the correct top-$N$ set.
\end{example}

\section{PHASE II\texorpdfstring{$\tfrac{1}{2}$}{1/2}: MUSICAL CHAIRS}
\label{sec:phase2half}

With a common set of $N$ high-quality target cells in hand, the players must perform the final coordination step: assigning themselves to these cells so that each is occupied by exactly one player. In our analysis, this common target set arises either from a consensus assumption, or from the public dither mechanism formalized in Appendix~G for the gap-free extension. This assignment is achieved via the Musical Chairs algorithm \citep{rosenski2016multi}. A key strength of our analysis is that we model the practical, challenging version of this protocol where players have no side-channel telling them which cells are already occupied; they must discover free cells through trial and error.

\subsection{The Seating Protocol}
The seating phase proceeds in rounds. Initially, all $N$ players are ``unseated." In each round, every unseated player samples a cell uniformly at random from the entire $N$-cell target set. A collision occurs at a cell if it is chosen by more than one unseated player, or if an unseated player chooses a cell that is already occupied by a seated player. If, however, a single unseated player chooses a currently unoccupied cell, that player becomes ``seated" at that cell. They cease to participate in the sampling process and will occupy that cell for the remainder of the horizon. The process terminates when all $N$ players are seated.

\subsection{Expected Seating Time and Regret}
Let $U_t$ be the number of unseated players at the start of round $t$. We define the expected number of players that become seated per round as "drift". When $u$ players remain unseated, this drift is given by $\Delta(u) := \EE[U_t - U_{t+1} \mid U_t=u] = \frac{u^2}{N}(1-\frac{1}{N})^{u-1}$. The quadratic dependence on $u$ means that progress is rapid when many players are searching for spots. This positive drift allows us to bound the total expected time for the dance to conclude.

\begin{theorem}[Expected Seating Time]
\label{thm:mc-time}
The expected time, $T_{MC}$, for all $N$ players to become seated is linear in the number of players: $\EE[T_{MC}] = O(N)$.
\end{theorem}

This result demonstrates that this simple, decentralized procedure is highly efficient and scales gracefully, far better than a naive $O(N \log N)$ coupon-collector analysis might suggest. The total regret incurred during this phase is therefore a fixed cost, dependent on $N$ but crucially, independent of the total time horizon $T$. This confirms that the entire coordination and seating process can be completed for a small, one-time price (not dependent on the time horizon).

\begin{corollary}[Seating-Phase Regret]
\label{cor:mc-regret}
The expected cumulative regret from Musical Chairs is bounded by a horizon-independent constant: $\EE[R_{MC}] = O(N^2)$.
\end{corollary}

\section{PHASE III: OPTIMIZATION WITHIN CELLS}
\label{sec:phase3}

With the completion of Phase II, the multi-agent coordination problem is solved. Each player is now the sole occupant of a distinct cell. For the remainder of the horizon, $T' = T - T_0 - T_1 - T_{MC}$, the problem decouples entirely into $N$ independent, single-player bandit problems. Collisions are no longer a concern, and each player's objective is simply to cultivate the maximum possible reward from within their own cell.

Each player must now solve a standard Lipschitz bandit problem over their assigned domain $C_j$. A provably near-optimal and standard approach for this task is an epoch-based, discretize-and-explore algorithm like Zooming \citep{Kleinberg}. In each epoch, they create a grid of points within their cell, with the grid resolution becoming finer over time.

The performance of such single-player strategies is well-established in the bandit literature \citep{bubeck2012regret,slivkins2019introduction}. The regret incurred by each player is known to follow the classical minimax rate for a $d$-dimensional Lipschitz problem.

\begin{proposition}[In-Cell Regret]
\label{prop:in-cell}
The expected regret for any player $j$ during the Phase III optimization within their cell $C_j$ over a duration of $T'$ rounds is bounded by:
\[
\EE\big[R_{\mathrm{in}}^{(C_j)}(T')\big] \;\le\; c_d\, (Lh)^{\frac{d}{d+2}} (T')^{\frac{d+1}{d+2}} + c'_d,
\]
where $c_d$ and $c'_d$ are constants that depend only on the dimension $d$.
\end{proposition}

The total regret from Phase III is the sum of these individual regrets across all $N$ players. This term represents the primary, horizon-dependent component of our overall regret bound.

\section{END-TO-END REGRET GUARANTEES}
\label{sec:global_guarantees}

With the components of our multi-phase protocol established, we combine them to get end-to-end performance guarantees. Our main result is that the significant upfront cost of decentralized coordination is carefully managed to ensure that the long-term performance is dictated by the optimal rate of single-agent learning. We build to this conclusion by first presenting a clean result for well-separated problem instances, and then stating our main guarantee that holds for any instance.

\subsection{Global Regret Under a Reward Gap}

Our first result quantifies the performance of the protocol in an ideal setting where the top-$N$ cells are clearly better than the rest. This scenario is formalized by the $\varepsilon$-uniqueness condition (Definition \ref{def:eps-uniq}), which assumes a sufficiently large ``reward gap" between the $N$-th best cell and the $(N+1)$-th best cell. When this gap exists, Phase II is guaranteed to lead to consensus, where all players identify the exact same set of top-$N$ cells. This allows for a seamless transition into seating and optimization.

\begin{theorem}[Global Regret with Consensus]
\label{thm:global}
Assume the $\varepsilon$-uniqueness condition holds (i.e., a sufficient reward gap exists). For any horizon $T$, the expected total regret is bounded by:
\begin{equation}
\label{eq:global-regret-bound}
\begin{aligned}
\EE[R_{cont}(T)] &\le \underbrace{N(T_0 + T_1) + c_{MC}N^2}_{\text{Coordination Cost}} \\
&\quad + \underbrace{c_d N (Lh)^{\frac{d}{d+2}} T^{\frac{d+1}{d+2}}}_{\text{Learning Cost}} + \delta_{sys}T,
\end{aligned}
\end{equation}
where $c_{MC}$ and $c_d$ are universal constants.
\end{theorem}

The theorem isolates the architecture of the protocol: an upfront coordination term followed by the standard single-player Lipschitz learning term. When \(\delta_{sys}\) is treated as a fixed confidence parameter, \(T_0\) and \(T_1\) are independent of \(T\). If one instead chooses \(\delta_{sys}=\delta_{sys}(T)\) to make the failure contribution \(\delta_{sys}T\) negligible in expectation, then the Phase~I/II radii inherit only extra logarithmic dependence on \(T\); equivalently, the coordination term becomes polylogarithmic in \(T\) rather than strictly horizon-independent. See Appendix~F, Remark (i) for details.

\subsection{A Gap-Free Guarantee}

The true strength of a decentralized protocol lies in its ability to perform well without favorable structural assumptions. A natural and challenging scenario arises when the reward gap is small or even zero, making the top-$N$ cells statistically indistinguishable from the next best. In this ``gap-free" setting, consensus is no longer guaranteed; different players might identify slightly different (but still high-quality) sets of target cells.

Our main result shows that our algorithm is robust to this challenge. This is achieved by running the protocol in epochs of doubling length ($T_k=2^k$), with the precision $\varepsilon_k$ recalibrated for each epoch. This standard ``doubling trick" \citep{cesa-bianchi2006prediction} allows the algorithm to control regret from potential sub-optimality without knowing the reward gaps or the horizon $T$ in advance. Our result relies on two additional ingredients, namely, public coverage/scheduling property together with a benign near-optimality-dimension condition. We state the resulting guarantee here and defer the formal assumptions and proof to Appendix~G. 

\begin{corollary}[Epochic, Gap-Free Global Regret]\label{cor:epochic-gapfree}
By running the multi-phase protocol in epochs, with precision $\varepsilon_k \propto 2^{-k/(d+2)}$ for epoch $k$, the algorithm achieves, for any $L$-Lipschitz reward function and any horizon $T$, an expected total regret of
\[
\EE[R_{\cont}(T)] \le \tilde{O}\left(N(Lh)^{\frac{d}{d+2}} T^{\frac{d+1}{d+2}}\right).
\]
\end{corollary}

Appendix~G also gives a fully gap-free single-shot baseline that holds without these additional assumptions, albeit with a weaker horizon exponent.

The above result demonstrates that by dynamically adapting the precision of the identification phases, our algorithm robustly achieves the minimax optimal regret rate without requiring any separation between the values of good and bad cells. The costs of repeated coordination in each epoch are controlled and are ultimately subsumed by the dominant learning cost.

\section{DISTANCE-THRESHOLD COLLISIONS}
\label{sec:packing}

Our analysis has centered on a partition-based collision model, a useful abstraction for systems with predefined operational zones. This section demonstrates the modularity of our framework by showing how it naturally extends to a more physically-motivated model where collisions are governed by proximity. 

In the distance-threshold model, a collision occurs if any two players' actions $X_t^{(j)}$ and $X_t^{(j')}$ are within a distance $\rho > 0$ of each other. To handle this, we reduce the problem to our existing framework. We first discretize the space $\X$ into a set of ``safe centers" $\{z_1, \dots, z_M\}$ that form an $r$-packing with $r > \rho$. We then associate each center $z_i$ with a ``safe ball" $B_i$ of radius $\sigma < (r-\rho)/2$. By construction, any two points chosen from two different safe balls are guaranteed to be separated by a distance greater than $\rho$, making inter-ball collisions impossible.

Our entire multi-phase protocol (Phase I - III) can then be applied directly, treating the family of safe balls $\{B_i\}$ as if they were the cells of the partition. Players identify, seat themselves upon, and perform optimization within these balls. The performance guarantees translate directly, with regret measured against the optimal assignment to these safe balls, $\OPT_{\pack}(r,\sigma,N)$.

\begin{theorem}[Regret in the Distance-Threshold Model]
\label{thm:packing}
When applied to a set of safe balls derived from an $r$-packing, our protocol's expected total regret is bounded by an expression identical in form to that in Theorem \ref{thm:global}, with the partition parameters $(K, h)$ replaced by the packing parameters $(M, \sigma)$.
\end{theorem}

\section{A MINIMAX LOWER BOUND}
\label{sec:lower}

Having established upper bounds for the regimes above, we now record a lower bound on the unavoidable horizon dependence. The following result shows that no decentralized algorithm can improve on the single-player exponent $T^{(d+1)/(d+2)}$. In particular, it matches the dominant $T$-dependence attained by our upper bound in the consensus regime, showing that decentralization and collisions do not worsen the core statistical difficulty of the problem.

\begin{theorem}[Minimax Lower Bound]
\label{thm:lower}
For the decentralized multi-player bandit problem in a partition-based collision model, for any decentralized algorithm, there exists an $L$-Lipschitz mean-reward function such that the expected regret is bounded below by:
\[
\EE[R_{\cont}(T)] \ge c \cdot N (Lh)^{\frac{d}{d+2}} T^{\frac{d+1}{d+2}},
\]
where $c$ is a constant depending only on dimension $d$.
\end{theorem}

The proof relies on a standard reduction from the well-established lower bound for finite-armed bandits \citep{lattimore2020bandit,bubeck2012regret}. We construct a challenging reward function that embeds $N$ independent, hard, finite-armed bandit problems into $N$ disjoint cells.

\section{CONCLUSION}
\label{sec:conclusion}
We have introduced and provided the first end-to-end solution for the cooperative, decentralized multi-player stochastic bandit problem in continuous domains with hard collisions. Our central contribution is a modular, multi-phase protocol that requires no communication between players. The key insight is to decouple the problem into a horizon-independent coordination stage and a horizon-dependent learning stage. A crucial innovation is our maxima-directed identification phase, which performs a localized search to avoid systemic biases inherent in simpler center-based discretizations. Our analysis culminates in a near-optimal, gap-free regret bound of $\tilde{O}(T^{(d+1)/(d+2)})$, which matches the single-player minimax rate for Lipschitz bandits in the consensus regime. Our analysis assumes a fixed known number of synchronous players and a common known upper bound on the Lipschitz constant. Extending the framework to unknown smoothness, asynchronous starts, changing player populations, or delayed feedback remains an interesting direction for future work.

\bibliographystyle{apalike}
\bibliography{refs}

\newpage
\newpage

\newpage

\section*{Checklist}

\begin{enumerate}

  \item For all models and algorithms presented, check if you include:
  \begin{enumerate}
    \item A clear description of the mathematical setting, assumptions, algorithm, and/or model. [Yes]
    \item An analysis of the properties and complexity (time, space, sample size) of any algorithm. [Yes]
    \item (Optional) Anonymized source code, with specification of all dependencies, including external libraries. [Not Applicable]
  \end{enumerate}

  \item For any theoretical claim, check if you include:
  \begin{enumerate}
    \item Statements of the full set of assumptions of all theoretical results. [Yes]
    \item Complete proofs of all theoretical results. [Yes]
    \item Clear explanations of any assumptions. [Yes]     
  \end{enumerate}

  \item For all figures and tables that present empirical results, check if you include:
  \begin{enumerate}
    \item The code, data, and instructions needed to reproduce the main experimental results (either in the supplemental material or as a URL). [Yes]
    \item All the training details (e.g., data splits, hyperparameters, how they were chosen). [Yes]
    \item A clear definition of the specific measure or statistics and error bars (e.g., with respect to the random seed after running experiments multiple times). [Yes]
    \item A description of the computing infrastructure used. (e.g., type of GPUs, internal cluster, or cloud provider). [Yes]
  \end{enumerate}

  \item If you are using existing assets (e.g., code, data, models) or curating/releasing new assets, check if you include:
  \begin{enumerate}
    \item Citations of the creator If your work uses existing assets. [Not Applicable]
    \item The license information of the assets, if applicable. [Not Applicable]
    \item New assets either in the supplemental material or as a URL, if applicable. [Not Applicable]
    \item Information about consent from data providers/curators. [Not Applicable]
    \item Discussion of sensible content if applicable, e.g., personally identifiable information or offensive content. [Not Applicable]
  \end{enumerate}

  \item If you used crowdsourcing or conducted research with human subjects, check if you include:
  \begin{enumerate}
    \item The full text of instructions given to participants and screenshots. [Not Applicable]
    \item Descriptions of potential participant risks, with links to Institutional Review Board (IRB) approvals if applicable. [Not Applicable]
    \item The estimated hourly wage paid to participants and the total amount spent on participant compensation. [Not Applicable]
  \end{enumerate}

\end{enumerate}

\clearpage
\appendix
\thispagestyle{empty}
\onecolumn
\aistatstitle{Supplementary Material:\\ Multi-Agent Lipschitz Bandits}

\section*{Appendix A: Related Work}
\label{app:rel-work}

Multi-armed bandits \citep{slivkins2019introduction,lattimore2020bandit} have been studied extensively, which can be traced back to the foundational contributions of \citep{thompson1933likelihood,lai1985asymptotically}. This framework has proven remarkably versatile across a wide range of applications, including clinical trials \citep{gittins1}, recommendation systems \citep{Li_2010}, financial optimization \citep{brochu}, and modern incentivized learning methods \citep{fraz,wang,chakraborty2024incentivized,chakraborty2025incentivized}.

The literature on multi-agent bandits is vast; we highlight the main works introducing collisions, decentralization, and cooperative settings. In the discrete-arm case, \citep{rosenski2016multi} introduced the Musical Chairs protocol for decentralized coordination, later improved by \citep{wang2020optimal} with optimal regret guarantees. Variants with delayed feedback, fairness, or corruption have also been explored \citep{boursier2020survey,ghaffari2024robust}.

Our setting departs by considering a continuum of actions. Single-player Lipschitz bandits have been studied via discretization, zooming, and hierarchical methods \citep{magureanu2014lipschitz,bubeck2011xarmed,Kleinberg}. These works provide minimax rates but do not address multi-agent collisions.

Multi-agent bandits in continuous domains remain underexplored. Existing works often assume stronger structure (e.g., convexity, linearity) \citep{bambos}, or allow explicit communication \citep{landgren2020distributed}.

In contrast, to our knowledge, our work is the first to provide a fully decentralized, communication-free solution for multi-player Lipschitz bandits with hard collisions, achieving near-optimal regret with only a constant coordination overhead.

\section*{Appendix B: Proofs for Phase I}
\label{app:phase1}

We collect here the proofs of Lemmas~\ref{lem:counts} and~\ref{lem:means} and Proposition~\ref{prop:phase1-brackets}. Throughout, probabilities and expectations are taken with respect to all sources of randomness (players' exploration, collision process, and reward noise). We begin by restating the standing assumptions, notation, and two identities that are used repeatedly.

\paragraph{Standing assumptions and notation.}
\begin{itemize}
    \item The action domain is $\X=[0,1]^d$, partitioned into $\Pcal=\{C_1,\ldots,C_K\}$, where each $C\in\Pcal$ is a (closed) axis-aligned hypercube of side length $h$; the cells have disjoint interiors and cover $\X$. For each $C\in\Pcal$, we denote by $x_C$ its geometric center and by
    \[
        D_h \;:=\; h\sqrt{d}
    \]
    the Euclidean diameter of any such cell. (Boundary cells may be smaller; using $D_h$ is conservative and simplifies the presentation.)
    \item The unknown mean-reward function $\mu:\X\to[0,1]$ is $L$-Lipschitz w.r.t.\ $\|\cdot\|_2$:
    \[
        |\mu(x)-\mu(y)| \;\le\; L\|x-y\|_2\qquad\forall x,y\in\X.
    \]
    \item In Phase~I, at each round $t\in[T_0]$ and for each player $j\in[N]$ independently, a cell $C_t^{(j)}$ is sampled uniformly from $\Pcal$; the player probes the cell center $x_{C_t^{(j)}}$.
    \item A collision occurs in cell $C$ at round $t$ if at least two players sample $C$ in that round. If player $j$ is the unique occupant of $C$ at round $t$, a noisy reward $Y_t^{(j)}$ with mean $\mu(x_C)$ is observed; otherwise a null symbol $\bot$ is observed and no reward is recorded.
    \item We analyze the standard bounded-reward model: whenever a reward is observed, $Y_t^{(j)}\in[0,1]$ with $\mathbb E[Y_t^{(j)}\mid C_t^{(j)}=C,\ \text{no collision}]=\mu(x_C)$. This implies the centered noise is $\sigma$-sub-Gaussian with $\sigma\le\tfrac12$.
    \item Players observe whether a collision occurred (via $\bot$). The collision/missingness process depends only on the independent exploration choices and is independent of the reward noise.
\end{itemize}

\paragraph{Two basic identities.}
\begin{enumerate}
\item[(i)] \textbf{Single-round success probability.} For a fixed player $j$ and cell $C\in\Pcal$, the probability that $j$ samples $C$ and no other player samples $C$ in the same round is
\begin{equation}
\label{eq:pK}
p_K \;:=\; \frac{1}{K}\Bigl(1-\frac{1}{K}\Bigr)^{N-1}.
\end{equation}
\emph{Proof.} $\Pr[C_t^{(j)}=C]=1/K$ by uniform sampling; for each other player $j'\neq j$, $\Pr[C_t^{(j')}\neq C]=1-1/K$ independently; multiply over $N-1$ players.

\item[(ii)] \textbf{Center vs.\ maximum (Lipschitz geometry).} For any cell $C\in\Pcal$,
\begin{equation}
\label{eq:center-max-bracket}
\mu(x_C)\ \le\ \mu^\ast(C)\ \le\ \mu(x_C) + \frac{L}{2}D_h \;=\; \mu(x_C) + \frac{Lh\sqrt d}{2}.
\end{equation}
\emph{Proof.} For any $x\in C$, $\|x-x_C\|_2\le D_h/2$, hence $\mu(x)\le \mu(x_C)+L\|x-x_C\|_2\le \mu(x_C)+\tfrac L2 D_h$. Taking $\sup_{x\in C}$ yields the RHS; the LHS is trivial since $\mu^\ast(C)=\sup_{x\in C}\mu(x)\ge \mu(x_C)$.
\end{enumerate}

\medskip

We write $U_j(C,t):=\mathbf 1\{\,\text{player $j$ is the unique occupant of cell $C$ at round $t$}\,\}$. The success count over Phase~I is
\[
o_{j,C}(T_0)\;:=\;\sum_{t=1}^{T_0} U_j(C,t).
\]
When $o_{j,C}(T_0)\ge 1$, the empirical mean at the cell center is
\[
\widehat\mu^{(0)}_j(C)\;:=\;\frac{1}{o_{j,C}(T_0)}\sum_{t=1}^{T_0} U_j(C,t)\,Y_t^{(j)};
\]
For $o_{j,C}(T_0)=0$, we define $\widehat\mu^{(0)}_j(C):=0$ (this choice is harmless because $\mu(x_C)\in[0,1]$ and our confidence radius below will be $\ge 1$ in that case). We recall the Phase~I radii from the main text:
\begin{equation}
\label{eq:phase1-radius}
r^{(0)}_{j}(C) \;:=\; \sqrt{\frac{\beta_0}{2\,\max\{1,o_{j,C}(T_0)\}}}, 
\qquad
\beta_0 \;:=\; \log\!\frac{4 N K (T_0+1)}{\delta_{I}},
\end{equation}
and
\begin{equation}
\label{eq:phase1-brackets-def}
\LCB^{(0)}_j(C) \;:=\; \widehat{\mu}_{j}^{(0)}(C) - r^{(0)}_{j}(C),
\qquad
\UCB^{(0)}_j(C) \;:=\; \widehat{\mu}_{j}^{(0)}(C) + r^{(0)}_{j}(C) + \frac{L h\sqrt{d}}{2}.
\end{equation}

\subsection*{B.1\quad Success counts under collisions (Lemma~\ref{lem:counts})}

\begin{lemma}[Restatement of Lemma~\ref{lem:counts}]
Fix $\eta\in(0,1)$. With probability at least $1-\delta_{I}/2$, simultaneously for all $j\in[N]$ and $C\in \Pcal$,
\[
(1-\eta)\,T_0 p_K \;\le\; o_{j,C}(T_0) \;\le\; (1+\eta)\,T_0 p_K.
\]
\end{lemma}

\begin{proof}
Fix $j\in[N]$ and $C\in\Pcal$. By the sampling protocol, across rounds $t=1,\ldots,T_0$ the random variables $\{C_t^{(j)}\}_{t=1}^{T_0}$ are i.i.d., and for each fixed round the choices $\{C_t^{(j')}\}_{j'=1}^N$ are mutually independent. Therefore, for this fixed $(j,C)$, the indicators $\{U_j(C,t)\}_{t=1}^{T_0}$ are i.i.d.\ Bernoulli$(p_K)$ with $p_K$ given by~\eqref{eq:pK}. Consequently,
\[
o_{j,C}(T_0)=\sum_{t=1}^{T_0} U_j(C,t)\ \sim\ \mathrm{Bin}\!\left(T_0,p_K\right).
\]
Let $\mu:=\mathbb E[o_{j,C}(T_0)]=T_0 p_K$. The two-sided multiplicative Chernoff bound for binomial variables states that for any $\eta\in(0,1)$,
\[
\Pr\!\left(\,\big|o_{j,C}(T_0)-\mu\big|>\eta\mu\,\right)\ \le\ 2\exp\!\left(-\frac{\eta^2 \mu}{3}\right)
\;=\; 2\exp\!\left(-\frac{\eta^2 T_0 p_K}{3}\right).
\]
Applying a union bound over all $N$ players and all $K$ cells yields
\[
\Pr\!\left(\,\exists (j,C)\in [N]\times \Pcal:\ \big|o_{j,C}(T_0)-T_0 p_K\big|>\eta T_0 p_K \,\right)
\ \le\ 2 N K \exp\!\left(-\frac{\eta^2 T_0 p_K}{3}\right).
\]
Hence, if $T_0$ is chosen to satisfy
\begin{equation}
\label{eq:T0-prereq-counts}
2 N K \exp\!\left(-\frac{\eta^2 T_0 p_K}{3}\right)\ \le\ \frac{\delta_I}{2}
\quad\Longleftrightarrow\quad
T_0 \ \ge\ \frac{3}{\eta^2 p_K}\,\log\!\frac{4 N K}{\delta_I},
\end{equation}
then the claimed event holds with probability at least $1-\delta_I/2$. This is precisely the prerequisite on $T_0$ used later (see \S\ref{sec:choice-T0}).
\end{proof}

\subsection*{B.2\quad Anytime concentration for center means (Lemma~\ref{lem:means})}

\begin{lemma}[Restatement of Lemma~\ref{lem:means}]
With probability at least $1-\delta_{I}/2$, simultaneously for all $j\in[N]$ and $C\in \Pcal$,
\[
\big|\widehat{\mu}_j^{(0)}(C) - \mu(x_C)\big| \;\le\; r^{(0)}_{j}(C),
\qquad
r^{(0)}_{j}(C) \;=\; \sqrt{\frac{\beta_0}{2\,\max\{1,o_{j,C}(T_0)\}}},\ \ 
\beta_0 \;=\; \log\!\frac{4 N K (T_0+1)}{\delta_{I}} .
\]
\end{lemma}

\begin{proof}
Fix $(j,C)$ and let $n:=o_{j,C}(T_0)$. Condition on the sigma-field generated by all exploration choices and collisions, and on the realized value of $n$. 

Conditional on $n$, and because the reward noise is independent of the exploration/collision process (see assumptions), the $n$ observed rewards associated with $(j,C)$ are i.i.d., lie in $[0,1]$, and have mean $\mu(x_C)$. Denote their average by $\widehat\mu^{(0)}_j(C)$ (for $n=0$ we keep the convention $\widehat\mu^{(0)}_j(C)=0$).

\underline{Case $n\ge 1$.} Hoeffding's inequality for $[0,1]$-bounded i.i.d.\ variables yields, for any $\epsilon>0$,
\[
\Pr\!\left(\,\big|\widehat\mu^{(0)}_j(C)-\mu(x_C)\big|>\epsilon \ \middle|\ o_{j,C}(T_0)=n\right)\ \le\ 2\exp\!\left(-2 n \epsilon^2\right).
\]
Choosing $\epsilon=\sqrt{\beta_0/(2n)}$ gives
\begin{equation}
\label{eq:cond-failure}
\Pr\!\left(\,\big|\widehat\mu^{(0)}_j(C)-\mu(x_C)\big|>\sqrt{\beta_0/(2n)} \ \middle|\ o_{j,C}(T_0)=n\right)\ \le\ 2 e^{-\beta_0}.
\end{equation}

\underline{Case $n=0$.} Then $\widehat\mu^{(0)}_j(C)=0$ by definition. Since $\mu(x_C)\in[0,1]$, the event
\[
\left\{\ \big|\widehat\mu^{(0)}_j(C)-\mu(x_C)\big|>\sqrt{\beta_0/2}\ \right\}
\]
is \emph{impossible} whenever $\sqrt{\beta_0/2}\ge 1$. With our choice $\beta_0=\log\!\frac{4NK(T_0+1)}{\delta_I}$ and global budget $\delta_I\le \delta_{\rm sys}<1/4$, we have
\[
\beta_0 \ \ge\ \log\!\frac{4\cdot 1\cdot 1\cdot 2}{1/4} \;=\; \log(32) \;>\; 2,
\]
so indeed $\sqrt{\beta_0/2}>1$ and the $n=0$ failure probability equals $0$.

Let $E_{j,C}$ be the (unconditional) failure event for this $(j,C)$, namely
\[
E_{j,C} \;:=\; \left\{\ \big|\widehat{\mu}_j^{(0)}(C)-\mu(x_C)\big|\ >\ \sqrt{\frac{\beta_0}{2\,\max\{1,o_{j,C}(T_0)\}}}\ \right\}.
\]
By the law of total probability and the discussion above,
\[
\Pr(E_{j,C})
\;=\; \sum_{n=0}^{T_0} \Pr\!\left(E_{j,C}\ \middle|\ o_{j,C}(T_0)=n\right)\Pr\!\left(o_{j,C}(T_0)=n\right)
\;\le\; \sum_{n=1}^{T_0} 2 e^{-\beta_0}\,\Pr\!\left(o_{j,C}(T_0)=n\right)
\;\le\; 2 e^{-\beta_0}.
\]
Finally, apply a union bound over all players and cells:
\[
\Pr\!\left(\exists (j,C)\in[N]\times\Pcal:\ E_{j,C}\right)
\ \le\ 2 N K\, e^{-\beta_0}
\ =\ \frac{2 N K}{\frac{4 N K (T_0+1)}{\delta_I}}
\ =\ \frac{\delta_I}{2(T_0+1)}
\ \le\ \frac{\delta_I}{2}.
\]
Hence, with probability at least $1-\delta_I/2$, the desired deviation bound holds simultaneously for all $(j,C)$, i.e., the lemma.
\end{proof}

\subsection*{B.3\quad Phase-I maxima brackets (Proposition~\ref{prop:phase1-brackets})}

\begin{proposition}[Restatement of Proposition~\ref{prop:phase1-brackets}]
On an event of probability at least $1-\delta_{I}$, simultaneously for all $j\in[N]$ and $C\in \Pcal$,
\[
\LCB^{(0)}_j(C) \;\le\; \mu^\ast(C) \;\le\; \UCB^{(0)}_j(C),
\]
with $\LCB^{(0)}_j,\UCB^{(0)}_j$ as in~\eqref{eq:phase1-brackets-def}.
\end{proposition}

\begin{proof}
By Lemma~\ref{lem:means}, on an event $\mathcal E_{\text{means}}$ of probability at least $1-\delta_I/2$, we have for all $(j,C)$
\[
\mu(x_C)\ \in\ \Bigl[\widehat\mu^{(0)}_j(C)-r^{(0)}_j(C),\ \widehat\mu^{(0)}_j(C)+r^{(0)}_j(C)\Bigr].
\]
Combining this with ~\eqref{eq:center-max-bracket} gives, for every $(j,C)$,
\[
\widehat{\mu}_j^{(0)}(C) - r^{(0)}_j(C) \;\le\; \mu^\ast(C) \;\le\; \widehat{\mu}_j^{(0)}(C) + r^{(0)}_j(C) + \frac{L}{2}D_h,
\]
which is exactly $\LCB^{(0)}_j(C)\le \mu^\ast(C)\le \UCB^{(0)}_j(C)$ by~\eqref{eq:phase1-brackets-def}. Thus, the validity holds on $\mathcal E_{\text{means}}$.

For later phases we also record the success-count event $\mathcal E_{\text{counts}}$ of Lemma~\ref{lem:counts}, which holds with probability at least $1-\delta_I/2$ provided $T_0$ satisfies~\eqref{eq:T0-prereq-counts}. While $\mathcal E_{\text{counts}}$ is not needed for the validity of the intervals, it controls their \emph{width}. On the intersection
\[
\mathcal E_I \;:=\; \mathcal E_{\text{means}} \cap \mathcal E_{\text{counts}},
\]
the bracket validity still holds, and by a union bound
\[
\Pr(\mathcal E_I)\ \ge\ 1 - \frac{\delta_I}{2} - \frac{\delta_I}{2} \ =\ 1-\delta_I.
\]
This proves the proposition as stated.
\end{proof}

\subsection*{B.4\quad Choice of $T_0$ and bracket precision}
\label{sec:choice-T0}

We record explicit (mild) lower bounds on the Phase~I budget $T_0$ that guarantee both Lemma~\ref{lem:counts} and a desired bracket accuracy.

\paragraph{Prerequisite for Lemma~\ref{lem:counts}.}
As shown in~\eqref{eq:T0-prereq-counts}, for any fixed $\eta\in(0,1)$ it suffices to take
\begin{equation}
\label{eq:T0-Lemma1}
T_0 \ \ge\ \frac{3}{\eta^2 p_K}\,\log\!\frac{4 N K}{\delta_I}.
\end{equation}

\paragraph{Uniform accuracy target.}
Fix $\alpha\in(0,1)$. On $\mathcal E_{\text{counts}}$ we have $o_{j,C}(T_0)\ge (1-\eta)T_0 p_K$ for all $(j,C)$. Using~\eqref{eq:phase1-radius},
\[
r^{(0)}_j(C) \ \le\ \sqrt{\frac{\beta_0}{2(1-\eta)T_0 p_K}},
\qquad \beta_0=\log\!\frac{4NK(T_0+1)}{\delta_I}.
\]
Therefore a sufficient condition for $r^{(0)}_j(C)\le \alpha$ uniformly over $(j,C)$ on $\mathcal E_{\text{counts}}$ is
\begin{equation}
\label{eq:T0-accuracy-implicit}
T_0 \ \ge\ \frac{\beta_0}{2\alpha^2(1-\eta)\,p_K}.
\end{equation}

\paragraph{A convenient consolidated choice.}
Since $\beta_0\ge \log\!\frac{4NK}{\delta_I}$ (as $T_0+1\ge 1$), the single implicit requirement
\begin{equation}
\label{eq:T0-consolidated}
T_0 \ \ge\ \frac{\beta_0}{p_K}\cdot \max\!\left\{ \frac{1}{\alpha^2},\ 12 \right\}
\qquad\text{with }\ \beta_0=\log\!\frac{4 N K (T_0+1)}{\delta_{I}}
\end{equation}
implies both~\eqref{eq:T0-Lemma1} (taking $\eta=\tfrac12$ so that $\tfrac{3}{\eta^2}=12$) and~\eqref{eq:T0-accuracy-implicit} (since $\frac{1}{2(1-\eta)}=1$ for $\eta=\tfrac12$). The dependence on $T_0$ inside $\beta_0$ is benign and monotone; any $T_0$ large enough to satisfy~\eqref{eq:T0-consolidated} is acceptable. In our subsequent analysis we only need the validity of the Phase~I brackets (which already holds under Lemma~\ref{lem:means}); the accuracy parameter $\alpha$ can thus be taken moderate.

\bigskip

\section*{Appendix C: Proof Details for Phase II}
\label{app:phase2}

Throughout this appendix we condition on the Phase~I clean event $\mathcal E_I$ from Appendix~B:
(i) all Phase~I brackets are valid for all players and cells; and
(ii) the Phase~I count concentration holds when invoked.
We keep the Phase~II failure budget $\delta_{II}\in(0,1/4)$ and allocate it explicitly across events.

\paragraph{Standing assumptions and notation (Phase II).}
\begin{itemize}
\item As in Appendix~B, observed rewards lie in $[0,1]$. Hence the centered noise is $\sigma$-sub-Gaussian with $\sigma\le \tfrac12$, and Hoeffding-type concentration bounds apply.
\item \textbf{Collision observability} Players observe a collision bit (null symbol $\bot$). The missingness process (collisions) depends only on the independent exploration choices and is independent of the reward noise. We will condition on counts of successful observations without affecting noise distributions.
\item \textbf{Active sets (safe elimination).} For player $j$, let $\theta^{(0)}_j$ be the $N$-th order statistic of $\{\LCB^{(0)}_j(C):C\in\Pcal\}$, and define
\begin{equation}
\label{eq:active-set-def}
\Sact^{(0)}_j \;:=\; \Bigl\{\,C\in \Pcal \,:\, \UCB^{(0)}_j(C)\ \ge\ \theta^{(0)}_j \Bigr\}.
\end{equation}
We will show that on $\mathcal E_I$ the active set is  {safe} and has size at least $N$ for every player.
\item \textbf{Probe nets.} For each cell $C$, fix an $\eta$-net $Z_C\subset C$ built as an axis-aligned grid of spacing $s:=\eta/\sqrt d$ in each coordinate. Then for every $x\in C$ there exists $z\in Z_C$ with $\|x-z\|_\infty\le s/2$ and hence $\|x-z\|_2\le (s/2)\sqrt d=\eta/2<\eta$. Its cardinality satisfies
\begin{equation}
\label{eq:Pmax}
|Z_C|\ \le\ \Bigl(2+\frac{h\sqrt d}{\eta}\Bigr)^{\!d}
\ \le\ c_d\,\Bigl(1+\frac{h}{\eta}\Bigr)^{\!d}
\ \le\ c'_d\,\Bigl(1+\Bigl(\frac{h}{\eta}\Bigr)^{\!d}\Bigr)
\ \le\ C_d\,\Bigl(\frac{h}{\eta}\Bigr)^{\!d},
\end{equation}
where $c_d:=(2+\sqrt d)^{d}$ and $C_d:=\max\{c_d,2^d\}$ is an absolute constant depending only on $d$. For simplicity we use the conservative bound $|Z_C|\le C_d (h/\eta)^d$ below.
Define $P_{\max}:=\max_{C\in\Pcal}|Z_C|$, so $P_{\max}\le C_d (h/\eta)^d$.
\item \textbf{Phase II sampling protocol (fixed for $T_1$ rounds).}
At each round $t=1,\dots,T_1$, independently across rounds and players:
\begin{enumerate}
\item Player $j$ samples a cell $C_t^{(j)}$ uniformly from $\Sact^{(0)}_j$.
\item Given $C_t^{(j)}$, player $j$ samples a probe $Z_t^{(j)}$ uniformly from $Z_{C_t^{(j)}}$ and probes $Z_t^{(j)}$.
\item A collision in a cell $C$ occurs if at least two players select any probes in the same $C$ at round $t$.
A probe is successful if no other player selects cell $C$ in that round.
\end{enumerate}
The active sets $\{\Sact^{(0)}_j\}$ and nets $\{Z_C\}$ are fixed throughout Phase~II.
\item \textbf{Counting notation.} For a triple $(j,C,z)$ with $C\in \Sact^{(0)}_j$ and $z\in Z_C$, define the per-round success indicator
\[
V_{j,C,z}(t)\ :=\ \mathbf 1\{\, C_t^{(j)}=C,\ Z_t^{(j)}=z,\ \text{and no other player selects }C\text{ at round }t\,\},
\]
and total successes $s_{j,C,z}(T_1):=\sum_{t=1}^{T_1} V_{j,C,z}(t)$.
Let
\[
\NumProbes \ :=\ \sum_{j=1}^N\ \sum_{C\in \Sact^{(0)}_j}\ |Z_C|
\quad\text{and}\quad
M_{act} \ :=\ \max_{j\in[N]} |\Sact^{(0)}_j|.
\]
\end{itemize}

\subsection*{C.1.\; Safe active sets (size and completeness)}

\begin{lemma}[Active-set safety and size]\label{lem:active-safety}
On $\mathcal E_I$, for every player $j$:
\begin{enumerate}
\item $|\Sact^{(0)}_j|\ \ge\ N$;
\item every true top-$N$ cell belongs to $\Sact^{(0)}_j$, i.e., if $T^\star$ denotes the set of $N$ cells with the largest $\mu^\ast(\cdot)$ values, then $T^\star\subseteq \Sact^{(0)}_j$.
\end{enumerate}
\end{lemma}

\begin{proof}
On $\mathcal E_I$, $\LCB^{(0)}_j(C)\le \mu^\ast(C)\le \UCB^{(0)}_j(C)$ for every $C$.

Let $\mu^\ast_{(1)}\ge\cdots\ge \mu^\ast_{(K)}$ be the sorted cell maxima and $T^\star$ the set of corresponding cells with ranks $1$ to $N$.

Since each $\LCB^{(0)}_j(C)\le \mu^\ast(C)$, the $N$-th LCB order statistic satisfies
$\theta^{(0)}_j \le \mu^\ast_{(N)}$.

For any $C\in T^\star$, we have $\UCB^{(0)}_j(C)\ge \mu^\ast(C)\ge \mu^\ast_{(N)}\ge \theta^{(0)}_j$,
hence $C\in \Sact^{(0)}_j$ by \eqref{eq:active-set-def}.

Therefore $T^\star\subseteq \Sact^{(0)}_j$ and $|\Sact^{(0)}_j|\ge |T^\star|=N$.
\end{proof}

\subsection*{C.2.\; Coverage: uniform lower bound on per-round probe success}

\begin{lemma}[Per-round success probability]\label{lem:q-lower}
Under the Phase~II protocol and Lemma~\ref{lem:active-safety}, for every triple $(j,C,z)$ and every round $t$,
\[
\Pr\bigl[V_{j,C,z}(t)=1\bigr]\ \ge\ q_{M_{act},\eta}
\quad\text{with}\quad
q_{M_{act},\eta}\ :=\ \frac{1}{M_{act}\,P_{\max}}\left(1-\frac{1}{N}\right)^{N-1},
\]
where $P_{\max}:=\max_C |Z_C| \le C_d (h/\eta)^d$.
\end{lemma}

\begin{proof}
Fix $(j,C,z)$ and $t$.

Player $j$ picks $C$ with probability $1/|\Sact^{(0)}_j|\ge 1/M_{act}$ and, given $C$, picks $z$ with probability $1/|Z_C|\ge 1/P_{\max}$.

For each other player $k\neq j$, either $C\notin \Sact^{(0)}_k$ (so $\Pr(C_t^{(k)}=C)=0$) or $C\in \Sact^{(0)}_k$ and then
$\Pr(C_t^{(k)}=C)=1/|\Sact^{(0)}_k|\le 1/N$ by Lemma~\ref{lem:active-safety}.

Independence across players implies
\[
\Pr\bigl[\text{no other player selects }C\bigr]\ \ge\ \Bigl(1-\frac{1}{N}\Bigr)^{N-1}.
\]

Multiplying completes the proof.
\end{proof}

\begin{lemma}[Phase-II Probe Coverage]\label{lem:coverage}
Let $b\in\mathbb N$ and $T_1\in\mathbb N$ satisfy
\[
b \ \ge\ 4\log\frac{2\,\NumProbes}{\delta_{II}}
\qquad\text{and}\qquad
T_1 \ \ge\ \frac{2b}{q_{M_{act},\eta}}.
\]
Then, with probability at least $1-\delta_{II}/2$, every triple achieves at least $b$ successes:
\[
\min_{(j,C,z)}\ s_{j,C,z}(T_1)\ \ge\ b.
\]
\end{lemma}

\begin{proof}
For any fixed triple, $\{V_{j,C,z}(t)\}_{t=1}^{T_1}$ are i.i.d.\ Bernoulli with success probability $q_{j,C,z}:=\Pr[V_{j,C,z}(t)=1]\ge q_{M_{act},\eta}$ by Lemma~\ref{lem:q-lower}. 

Thus
$\mathbb E[s_{j,C,z}(T_1)]=T_1 q_{j,C,z}\ge 2b$.

By the Chernoff lower-tail bound with $\lambda=\tfrac12$,
\[
\Pr\bigl[s_{j,C,z}(T_1)< b\bigr]\ \le\ \exp\!\Bigl(-\frac{\mathbb E[s_{j,C,z}(T_1)]}{8}\Bigr)\ \le\ e^{-b/4}.
\]

A union bound over all $\NumProbes$ triples gives
$\Pr\bigl(\exists (j,C,z): s_{j,C,z}(T_1)< b\bigr)\le \NumProbes\, e^{-b/4}\le \delta_{II}/2$ by the choice of $b$.
\end{proof}

\subsection*{C.3.\; Uniform probe accuracy}

\begin{lemma}[Uniform probe accuracy]\label{lem:probe-accuracy}
Let
\[
\beta_1\ :=\ \log\frac{4\,\NumProbes}{\delta_{II}}
\qquad\text{and}\qquad
r_1\ :=\ \sqrt{\frac{\beta_1}{2b}}.
\]
On the event of Lemma~\ref{lem:coverage} (so $s_{j,C,z}(T_1)\ge b$ for all triples), we have
\[
\Pr\!\left[\ \exists (j,C,z):\ \big|\widehat\mu_j(C,z)-\mu(z)\big|> r_1\ \right]\ \le\ \frac{\delta_{II}}{2},
\]
where $\widehat\mu_j(C,z)$ is the empirical mean over the $s_{j,C,z}(T_1)$ successful observations at $(j,C,z)$.
Equivalently, with probability at least $1-\delta_{II}/2$,
\[
\big|\widehat\mu_j(C,z)-\mu(z)\big|\ \le\ r_1
\qquad\text{simultaneously for all }(j,C,z).
\]
\end{lemma}

\begin{proof}
Fix $(j,C,z)$. Condition on $s:=s_{j,C,z}(T_1)\ge b$ and on the exploration/collision sigma-field.

By our assumption on the collisions, the $s$ observed rewards at $(j,C,z)$ are i.i.d.\ in $[0,1]$ with mean $\mu(z)$.
For any $n\ge b$, Hoeffding yields
\[
\Pr\!\left(\big|\widehat\mu_j(C,z)-\mu(z)\big|>\sqrt{\beta_1/(2n)}\ \middle|\ s=n\right)\ \le\ 2e^{-\beta_1}.
\]

Hence, by the law of total probability restricted to $n\ge b$,
\[
\Pr\!\left(\big|\widehat\mu_j(C,z)-\mu(z)\big|>r_1\ \middle|\ s\ge b\right)\ \le\ 2e^{-\beta_1}.
\]

A union bound over the $\NumProbes$ triples yields the claim with the chosen $\beta_1$.
\end{proof}

\begin{remark}[Anytime variant (not used in parameterization)]
\label{rem:anytime}
If one prefers a radius that adapts to the realized count, set
$\beta_1^{\mathrm{any}}:=\log\frac{4\,\NumProbes\,(T_1+1)}{\delta_{II}}$
and $r^{\mathrm{any}}(s):=\sqrt{\beta_1^{\mathrm{any}}/(2\max\{1,s\})}$.
Then, by the same Hoeffding+\;law-of-total-probability argument and a union bound over $n\in\{1,\dots,T_1\}$,
\[
\Pr\!\left(\ \exists (j,C,z):\ \big|\widehat\mu_j(C,z)-\mu(z)\big|> r^{\mathrm{any}}(s_{j,C,z}(T_1))\ \right)\ \le\ \frac{\delta_{II}}{2}.
\]
We do not use this form below; our fixed-$b$ parameterization already yields the target width with fewer log factors.
\end{remark}

\subsection*{C.4.\; Refined maxima brackets}

For player $j$ and active cell $C\in \Sact^{(0)}_j$, define
\begin{equation}
\label{eq:phase2-brackets-def}
\LCB^{(1)}_j(C)\ :=\ \max_{z\in Z_C}\bigl\{\widehat\mu_j(C,z)-r_1\bigr\},
\qquad
\UCB^{(1)}_j(C)\ :=\ \max_{z\in Z_C}\bigl\{\widehat\mu_j(C,z)+r_1\bigr\} + L\eta.
\end{equation}

\begin{proposition}[Refined Maxima Brackets]\label{prop:refined-bracket}
On the intersection of the events in Lemmas~\ref{lem:coverage} and~\ref{lem:probe-accuracy} (probability at least $1-\delta_{II}$), simultaneously for all players $j$ and active cells $C\in \Sact^{(0)}_j$,
\[
\LCB^{(1)}_j(C)\ \le\ \mu^\ast(C)\ \le\ \UCB^{(1)}_j(C)
\quad\text{and}\quad
\UCB^{(1)}_j(C)-\LCB^{(1)}_j(C)\ \le\ 2r_1 + L\eta.
\]
\end{proposition}

\begin{proof}
Work on the event of Lemma~\ref{lem:probe-accuracy}. 

For any $z\in Z_C$,
$\mu(z)\in [\widehat\mu_j(C,z)-r_1,\ \widehat\mu_j(C,z)+r_1]$.

Taking maxima over $z\in Z_C$ gives
\[
\max_{z\in Z_C}\mu(z)\ \in\ \left[\ \max_{z\in Z_C}\bigl(\widehat\mu_j(C,z)-r_1\bigr),\ \max_{z\in Z_C}\bigl(\widehat\mu_j(C,z)+r_1\bigr)\ \right].
\]

By the $\eta$-net property and $L$-Lipschitzness, $\mu^\ast(C)\in [\,\max_{z\in Z_C}\mu(z),\ \max_{z\in Z_C}\mu(z)+L\eta\,]$.

Combining yields the validity and the width bound stated.
\end{proof}

\subsection*{C.5.\; Selecting the top-$N$ cells: gap-free $\varepsilon$-optimality and consensus}

Recall that each player $j$ selects the $N$ cells with the largest $\LCB^{(1)}_j(C)$
(breaking ties by a fixed public ordering), and denote the selected set by $S^{(j)}_\varepsilon$.

\begin{theorem}[Gap-Free $\varepsilon$-Optimality]\label{thm:gapfree-safety}
Let $\varepsilon>0$. If Phase~II parameters are such that $2r_1+L\eta\le \varepsilon$,
then on the intersection of the events in Lemmas~\ref{lem:coverage} and~\ref{lem:probe-accuracy} (probability at least $1-\delta_{II}$), for every player $j$ and every $C\in S^{(j)}_\varepsilon$,
\[
\mu^\ast(C)\ \ge\ \mu^\ast_{(N)} - \varepsilon.
\]
\end{theorem}

\begin{proof}
On the event of Proposition~\ref{prop:refined-bracket}, for any true top-$N$ cell $C^\star$ we have
$\LCB^{(1)}_j(C^\star)\ge \mu^\ast(C^\star)-\varepsilon\ge \mu^\ast_{(N)}-\varepsilon$.

Let $\theta^{(1)}_j$ be the $N$-th order statistic of $\{\LCB^{(1)}_j(C):C\in \Sact^{(0)}_j\}$; since $T^\star\subseteq \Sact^{(0)}_j$ (Lemma~\ref{lem:active-safety}) and there are $N$ cells with LCB at least $\mu^\ast_{(N)}-\varepsilon$, we have $\theta^{(1)}_j\ge \mu^\ast_{(N)}-\varepsilon$.

If $C\in S^{(j)}_\varepsilon$, then $\LCB^{(1)}_j(C)\ge \theta^{(1)}_j$, and validity gives
$\mu^\ast(C)\ge \LCB^{(1)}_j(C)\ge \mu^\ast_{(N)}-\varepsilon$.
\end{proof}

\begin{definition}[$\varepsilon$-Uniqueness at the Top-$N$]\label{def:eps-uniq}
A mean function $\mu$ is $\varepsilon$-unique at the top-$N$ if there exists a unique set $S^\dagger\subset \Pcal$ of size $N$ such that
\[
\min_{C\in S^\dagger} \mu^\ast(C) \ \ge\ \max_{C\notin S^\dagger} \mu^\ast(C) + 2\varepsilon.
\]
\end{definition}

\begin{lemma}[Consensus under $\varepsilon$-Uniqueness]\label{lem:consensus}
Assume $\varepsilon$-uniqueness at the top-$N$ and that all refined brackets have width at most $\varepsilon$ and contain $\mu^\ast(C)$.
Then $S^{(j)}_\varepsilon = S^\dagger$ for all $j\in[N]$.
\end{lemma}

\begin{proof}
For any $C_{\rm good}\in S^\dagger$ and $C_{\rm bad}\notin S^\dagger$,
\[
\LCB^{(1)}_j(C_{\rm good})\ \ge\ \mu^\ast(C_{\rm good})-\varepsilon\ \ge\ \mu^\ast(C_{\rm bad})+\varepsilon\ \ge\ \LCB^{(1)}_j(C_{\rm bad}),
\]
using $\varepsilon$-uniqueness and $\LCB^{(1)}_j(C_{\rm bad})\le \mu^\ast(C_{\rm bad})$.
Thus the $N$ largest LCBs are attained exactly on $S^\dagger$ (ties do not change this inclusion).
\end{proof}

\subsection*{C.6.\; Parameter choices and complexity bounds}

We summarize our parameter choices ensuring width $\le \varepsilon$ with probability at least $1-\delta_{II}$.

\paragraph{Cardinalities and probabilities.}
By \eqref{eq:Pmax} and Lemma~\ref{lem:active-safety},
\[
\NumProbes\ =\ \sum_{j=1}^N \sum_{C\in \Sact^{(0)}_j} |Z_C|
\ \le\ N\,M_{act}\,P_{\max}
\ \le\ N\,M_{act}\,C_d\Bigl(\frac{h}{\eta}\Bigr)^{\!d}.
\]
Lemma~\ref{lem:q-lower} gives
\[
q_{M_{act},\eta}\ \ge\ \frac{1}{M_{act}\,P_{\max}}\left(1-\frac{1}{N}\right)^{N-1}
\ \ge\ \frac{1}{M_{act}\,C_d}\left(\frac{\eta}{h}\right)^{\!d}\left(1-\frac{1}{N}\right)^{N-1}.
\]

\paragraph{Concrete schedule for a target $\varepsilon>0$.}
Choose
\[
\eta\ =\ \frac{\varepsilon}{2L},\qquad
\beta_1\ =\ \log\frac{4\,\NumProbes}{\delta_{II}},\qquad
b\ \ge\ \max\!\left\{\, 4\log\frac{2\,\NumProbes}{\delta_{II}},\ \frac{8\,\beta_1}{\varepsilon^2}\,\right\},
\]
so that $2r_1+L\eta \le \varepsilon$, where $r_1=\sqrt{\beta_1/(2b)}$.
Then, using Lemma~\ref{lem:coverage} and the bound on $q_{M_{act},\eta}$,
\[
T_1 \ \ge\ \frac{2b}{q_{M_{act},\eta}}
\ \le\ \frac{2\,b\,M_{act}\,P_{\max}}{\left(1-\frac{1}{N}\right)^{N-1}}
\ \le\ \frac{2\,b\,M_{act}\,C_d}{\left(1-\frac{1}{N}\right)^{N-1}}
\left(\frac{h}{\eta}\right)^{\!d}.
\]
With $\eta=\varepsilon/(2L)$ and $b=\Theta\!\big(\beta_1/\varepsilon^2 + \log(\NumProbes/\delta_{II})\big)$ we obtain
\[
T_1 \ =\ O\!\left(
\frac{M_{act}\,(Lh)^d}{\varepsilon^{d+2}}\cdot
\Bigl[\beta_1 + \log\frac{\NumProbes}{\delta_{II}}\Bigr]
\right),
\qquad
\NumProbes \ \le\ O\!\Big(N\,M_{act}\,(Lh/\varepsilon)^d\Big).
\]
Thus the Phase~II budget $T_1$ is independent of the horizon $T$ and scales (up to polylog factors) as
\[
T_1 \ =\ \tilde O\!\Big(M_{act}\,(Lh)^d\,\varepsilon^{-(d+2)}\Big).
\]
If desired, one may replace $M_{act}$ by $K$ using $M_{act}\le K$ for a looser but simpler bound.
\medskip

\noindent\textbf{Failure budget aggregation.}
By Lemmas~\ref{lem:coverage} and~\ref{lem:probe-accuracy}, the intersection of the coverage and uniform-accuracy events holds with probability at least $1-\delta_{II}$. Proposition~\ref{prop:refined-bracket} and Theorem~\ref{thm:gapfree-safety} are stated on this intersection. Lemma~\ref{lem:consensus} is a deterministic consequence of the refined brackets and $\varepsilon$-uniqueness.

\medskip

\noindent\textbf{Remarks.}
(i) The constant $L\eta$ in \eqref{eq:phase2-brackets-def} can be tightened to $L\eta/2$ using the $\eta/2$ covering radius; we keep $L\eta$ for simplicity and monotonicity with respect to $\eta$.
(ii) If one prefers an anytime-in-$n$ accuracy with an adaptive radius, use Remark~\ref{rem:anytime}; this modifies $\beta_1$ to include a $\log(T_1\!+\!1)$ factor without changing asymptotics in $\varepsilon$.
(iii) The protocol fixes active sets during Phase~II; allowing adaptive shrinking would require re-deriving $q_{M_{act},\eta}$ or freezing a lower bound on $|\Sact^{(0)}_j|$ throughout the phase.

\bigskip

\section*{Appendix D: Proof Details for Phase II$\tfrac{1}{2}$ (Musical Chairs)}
\label{app:mc}

We now provide proofs for the Musical Chairs (MC) seating phase.

We know that Phase II concluded such that all players share the same $N$-cell target set (e.g., because the $\epsilon$-uniqueness condition holds). The analysis below is conditioned on this event.

There are $N$ target cells and $N$ players. At the start of the phase, all players are unseated.
In each round, every unseated player independently samples a cell uniformly at random from the $N$-cell target set.
If exactly one unseated player chooses a currently unoccupied cell, that player becomes seated at that cell and remains seated thereafter.
If a cell is chosen by two or more unseated players, or by any unseated player together with an already seated player, a collision occurs at that cell and all players who chose it obtain zero reward in that round (the seated player remains seated).

Let $U_t\in\{0,1,\dots,N\}$ be the number of unseated players at the start of round $t$.
Then exactly $N-U_t$ cells are occupied (by seated players) and exactly $U_t$ cells are free.
Define the  {drift} at state $u$ by
\begin{equation}
\label{eq:mc-drift-def}
\Delta(u)\ :=\ \EE\!\left[\,U_t-U_{t+1}\ \middle|\ U_t=u\,\right],
\end{equation}
i.e., the expected number of newly seated players in a round when $u$ players are currently unseated.

\subsection*{D.1.\; Drift formula and basic bounds}

\begin{lemma}[Exact drift and a uniform lower bound]
\label{lem:mc-drift}
For every $u\in\{1,\dots,N\}$,
\begin{equation}
\label{eq:mc-drift}
\Delta(u)\ =\ \frac{u^2}{N}\left(1-\frac{1}{N}\right)^{u-1}
\ \ \ge\ \ \frac{u^2}{e\,N}.
\end{equation}
\end{lemma}

\begin{proof}
When $u$ players are unseated, there are $u$ free cells. 

Fix a particular free cell.

The probability that  {exactly one} of the $u$ unseated players chooses this cell equals
$u\cdot \tfrac{1}{N}\cdot (1-\tfrac{1}{N})^{u-1}$: choose which unseated player (there are $u$ choices), that player picks the cell with probability $1/N$, and each of the remaining $u-1$ unseated players avoids it with probability $(1-1/N)$, independently.

Each free cell with exactly one chooser yields exactly one newly seated player, and free cells are disjoint, so by linearity of expectation over the $u$ free cells we obtain
\[
\Delta(u)\ =\ u\cdot \Bigl[u\cdot \tfrac{1}{N}\cdot (1-\tfrac{1}{N})^{u-1}\Bigr]\ =\ \frac{u^2}{N}\left(1-\frac{1}{N}\right)^{u-1}.
\]

For the lower bound, $(1-1/N)^{u-1}\ge (1-1/N)^{N-1}\ge e^{-1}$ for all $N\ge 1$, yielding $\Delta(u)\ge u^2/(eN)$.
\end{proof}

\begin{lemma}[Monotonicity of the drift]
\label{lem:mc-drift-monotone}
The function $u\mapsto \Delta(u)$ is strictly increasing on $\{1,2,\dots,N\}$.
\end{lemma}

\begin{proof}
For $u\in\{1,\dots,N-1\}$,
\[
\frac{\Delta(u+1)}{\Delta(u)}
\ =\ \frac{(u+1)^2}{u^2}\cdot \left(1-\frac{1}{N}\right)
\ =\ \left(1+\frac{1}{u}\right)^2\left(1-\frac{1}{N}\right).
\]

Since $u\le N-1$, we have $1+\frac{1}{u}\ \ge\ 1+\frac{1}{N-1}\ >\ 1$ and hence
\[
\left(1+\frac{1}{u}\right)^2\left(1-\frac{1}{N}\right)
\ \ge\ \left(1+\frac{1}{N-1}\right)^2\left(1-\frac{1}{N}\right)
\ =\ \frac{N^2}{(N-1)^2}\cdot \frac{N-1}{N}
\ =\ \frac{N}{N-1}\ >\ 1.
\]

Thus $\Delta(u+1)>\Delta(u)$.
\end{proof}

\subsection*{D.2.\; Expected seating time}

We now prove that the expected number of rounds to seat all players is linear in $N$.

\begin{theorem}[Expected seating time]
\label{thm:mc-time-app}
Let $T_{MC}:=\inf\{t\ge 1: U_t=0\}$ be the (a.s. finite) stopping time at which all players are seated.
Then
\[
\EE[T_{MC}]\ \le\ \sum_{u=1}^{N}\frac{1}{\Delta(u)}\ \le\ \frac{e\pi^2}{6}\,N,
\]
and in particular $\EE[T_{MC}]=O(N)$.
\end{theorem}

\begin{proof}
Define the potential
\[
V(u)\ :=\ \sum_{v=1}^{u}\frac{1}{\Delta(v)},\qquad u\in\{0,1,\dots,N\}.
\]
Note that $V(0)=0$ and $V$ is finite and nondecreasing.

Let $Y_t:=U_t-U_{t+1}\in\{0,1,\dots,U_t\}$ denote the number of newly seated players in round $t$.

Fix $t$ and condition on $U_t=u>0$. Using the definition of $V$,
\[
V(U_t)-V(U_{t+1})\ =\ V(u)-V(u-Y_t)\ =\ \sum_{k=0}^{Y_t-1}\frac{1}{\Delta(u-k)}.
\]

By Lemma~\ref{lem:mc-drift-monotone}, $\Delta(\cdot)$ is increasing, so $1/\Delta(\cdot)$ is decreasing. 

Therefore
\[
V(U_t)-V(U_{t+1})\ \ge\ Y_t\cdot \frac{1}{\Delta(u)}.
\]

Taking conditional expectations and using the definition of the drift \eqref{eq:mc-drift-def},
\[
\EE\!\left[\, V(U_t)-V(U_{t+1}) \ \middle|\ U_t=u \,\right]
\ \ge\ \frac{\EE[Y_t\mid U_t=u]}{\Delta(u)}\ =\ \frac{\Delta(u)}{\Delta(u)}\ =\ 1.
\]

Taking expectations and summing over $t=0,1,\dots,T_{MC}-1$, we obtain (by the tower property; no optional-stopping theorem is required)
\[
\EE\!\left[\,V(U_0)-V(U_{T_{MC}})\,\right]
\ =\ \sum_{t=0}^{\infty} \EE\!\left[\,\mathbf 1\{t<T_{MC}\}\cdot \bigl(V(U_t)-V(U_{t+1})\bigr)\,\right]
\ \ge\ \sum_{t=0}^{\infty} \EE\!\left[\,\mathbf 1\{t<T_{MC}\}\,\right]
\ =\ \EE[T_{MC}].
\]
Since $U_0=N$ and $V(0)=0$, we have $\EE[T_{MC}]\le V(N)=\sum_{u=1}^{N}\frac{1}{\Delta(u)}$.

Finally, using Lemma~\ref{lem:mc-drift}, $\Delta(u)\ge u^2/(eN)$, hence
\[
\sum_{u=1}^{N}\frac{1}{\Delta(u)}\ \le\ eN \sum_{u=1}^{N}\frac{1}{u^2}\ \le\ \frac{e\pi^2}{6}\,N,
\]
which proves the claim.
\end{proof}

\subsection*{D.3.\; Seating-phase regret}

We evaluate the expected regret accumulated during the seating phase, measured against the per-round benchmark of obtaining $N$ unit-normalized rewards (one per target cell). 

\begin{corollary}[Seating-phase regret: a simple bound]
\label{cor:mc-regret}
Let $R_{MC}$ denote the cumulative regret incurred during the seating phase.
Then
\[
\EE[R_{MC}]\ \le\ N \cdot \EE[T_{MC}]\ \le\ \frac{e\pi^2}{6}\,N^2,
\]
so in particular $\EE[R_{MC}]=O(N^2)$. The bound is independent of the overall horizon $T$.
\end{corollary}

\begin{proof}
In each round, the per-round benchmark is at most $N$ (one unit per cell), and actual reward is nonnegative. Therefore the instantaneous regret is at most $N$.

By Linearity of expectation on the random-time sum,
\[
\EE[R_{MC}] \ =\ \EE\!\left[\sum_{t=1}^{T_{MC}} \text{regret}_t\right]
\ \le\ \EE\!\left[\sum_{t=1}^{T_{MC}} N\right]\ =\ N\,\EE[T_{MC}],
\]
and the claim follows from Theorem~\ref{thm:mc-time-app}.
\end{proof}

\begin{remark}[Sharper regret bound]
\label{rem:mc-regret-sharp}
The $O(N^2)$ bound is conservative. One can show $\EE[R_{MC}]=O(N\log N)$ as follows.
In a round with $u$ unseated players:
\begin{itemize}
\item The expected number of occupied cells that are hit by at least one unseated player is at most
$(N-u)\bigl(1-(1-1/N)^u\bigr)\ \le\ (N-u)\cdot \frac{u}{N}\ \le\ u$,
so expected regret contributed by  {occupied} cells is $\le u$ (each such collision zeros that cell’s reward).
\item Among the $u$ free cells, the expected number that are  {not} uniquely chosen by exactly one unseated player is
$u\left[1 - u\cdot \tfrac{1}{N}\cdot (1-\tfrac{1}{N})^{u-1}\right]\ \le\ u$,
so expected regret from  {free} cells is also $\le u$.
\end{itemize}

Thus $\EE[\text{regret}_t\mid U_t=u]\le 2u$. 

Summing over the seating phase,
\[
\EE[R_{MC}]\ \le\ 2\,\EE\!\left[\sum_{t=0}^{T_{MC}-1} U_t\right].
\]

Define the potential $G(u):=\sum_{v=1}^{u}\frac{v}{\Delta(v)}$. 

By the same drift argument as in Theorem~\ref{thm:mc-time-app}, one shows
$\EE\!\left[\sum_{t< T_{MC}} U_t\right]\ \le\ G(N)=\sum_{v=1}^{N}\frac{v}{\Delta(v)}$.

Using $\Delta(v)\ge v^2/(eN)$ gives $\sum_{v=1}^{N}\frac{v}{\Delta(v)}\le eN \sum_{v=1}^{N}\frac{1}{v}\le eN(1+\ln N)$.

Therefore $\EE[R_{MC}]=O(N\log N)$. We keep Corollary~\ref{cor:mc-regret} in the main text for simplicity, as it is horizon-independent and sufficient for our overall bounds.
\end{remark}

\bigskip

\section*{Appendix E: Proof of Proposition~\ref{prop:in-cell} (In-Cell Regret)}
\label{app:in-cell}

We prove the standard minimax rate for a single player optimizing within a fixed cell.
Throughout this appendix, rewards are bounded in $[0,1]$ (consistent with Appendix~B), and $\mu$ is $L$-Lipschitz on the hypercube cell $C\subset\mathbb R^d$ of side length $h$ in the Euclidean norm.

\paragraph{Reduction to the unit cube.}
Let $\phi:[0,1]^d\to C$ be the affine bijection $\phi(x')=x_0+h\,x'$ that maps the unit cube onto $C$ (for some cell origin $x_0$).
Define $\mu'(x'):=\mu(\phi(x'))$ for $x'\in[0,1]^d$.
Then, for any $x',y'\in[0,1]^d$,
\[
|\mu'(x')-\mu'(y')|
\;=\;
|\mu(\phi(x'))-\mu(\phi(y'))|
\;\le\; L\, \|\phi(x')-\phi(y')\|_2
\;=\; L h\, \|x'-y'\|_2.
\]
Thus $\mu'$ is $L'$-Lipschitz on $[0,1]^d$ with $L':=Lh$.

\paragraph{Known minimax rate on the unit cube.}
For the $d$-dimensional Lipschitz bandit on $[0,1]^d$ with Lipschitz constant $L'$, there exist algorithms (e.g., the Zooming algorithm \citep{Kleinberg}) whose expected regret over any horizon $T'\ge 1$ satisfies
\begin{equation}
\label{eq:unit-cube-rate}
\EE\big[R_{[0,1]^d}(T')\big]
\;\le\;
c_d\,(L')^{\frac{d}{d+2}}\, (T')^{\frac{d+1}{d+2}} + c'_d,
\end{equation}
where $c_d,c'_d$ depend only on $d$.
This result is classical; see, e.g., Kleinberg et al. \citep{Kleinberg} or Slivkins \citep{slivkins2019introduction}.

\paragraph{Combining.}
Applying \eqref{eq:unit-cube-rate} to $\mu'$ and substituting $L'=Lh$ gives
\[
\EE\big[R_{\mathrm{in}}^{(C)}(T')\big]
\;\le\;
c_d\,(Lh)^{\frac{d}{d+2}}\,(T')^{\frac{d+1}{d+2}} + c'_d,
\]
which is exactly Proposition~\ref{prop:in-cell}.
\qed

\medskip

\noindent\textbf{Remarks.}
(i) The proof above is intentionally short: it isolates the $h$-dependence via rescaling and then cites a standard unit-cube result. If desired, one may instantiate a concrete algorithm (e.g., Zooming) and track constants; this does not affect the rate or the $(Lh)^{d/(d+2)}$ dependence.
(ii) If one prefers finite-armed baselines, an epochic discretize-and-explore scheme on grids of mesh $\asymp (Lh)^{2/(d+2)}\,2^{e/(d+2)}$ per coordinate also yields the same rate by balancing exploration and discretization errors; we omit these routine details.

\section*{Appendix F: Global Regret Bounds (Theorem~\ref{thm:global})}
\label{app:global}

We assemble the end-to-end bound and reconcile expectation-level failure terms with the main-text statement. Throughout, per-round reward is in $[0,1]$ for each player; thus the per-round system benchmark (sum of the top-$N$ cell maxima) is at most $N$.

\subsection*{F.1.\; Clean event and failure budgeting}

Let $\mathcal{E}_I$ and $\mathcal{E}_{II}$ be the Phase~I and Phase~II success events as defined in Appendices~B and~C, respectively (valid brackets and coverage/accuracy).
Set
\[
\delta_{I}\ =\ \delta_{II}\ =\ \frac{\delta_{sys}}{2N}.
\]
By union bounds across players and cells (already accounted for in Appendices~B and~C), we have
\[
\Pr(\mathcal{E}_I)\ \ge\ 1-\frac{\delta_{sys}}{2N},
\qquad
\Pr(\mathcal{E}_{II})\ \ge\ 1-\frac{\delta_{sys}}{2N}.
\]
Define the global clean event $\mathcal E:=\mathcal E_I\cap \mathcal E_{II}$ for which
\begin{equation}
\label{eq:clean-event-budget}
\Pr(\mathcal E)\ \ge\ 1-\frac{\delta_{sys}}{N}.
\end{equation}
We write $R_{\cont}(T)$ for the total regret up to time $T$ and decompose its expectation by indicators of $\mathcal E$:
\[
\EE[R_{\cont}(T)]
\;=\;
\EE\big[R_{\cont}(T)\,\mathbf 1\{\mathcal E\}\big]
+\EE\big[R_{\cont}(T)\,\mathbf 1\{\mathcal E^c\}\big].
\]
Since per-round regret is at most $N$, we have
\begin{equation}
\label{eq:failure-expectation}
\EE\big[R_{\cont}(T)\,\mathbf 1\{\mathcal E^c\}\big]
\ \le\ N T \,\Pr(\mathcal E^c)
\ \le\ \delta_{sys}\, T,
\end{equation}
which matches the main-text term.

\subsection*{F.2.\; Identification and seating terms (conditioned on $\mathcal E$)}

Condition on $\mathcal E$. During Phase~I and Phase~II, each round contributes at most $N$ to regret:
\begin{equation}
\label{eq:phaseI-II}
\EE\!\left[R_{I}(T_0)+R_{II}(T_1)\,\big|\,\mathcal E\right] \ \le\ N(T_0+T_1).
\end{equation}
We assume the $\varepsilon$-uniqueness condition of Definition~\ref{def:eps-uniq} (main text) in Theorem~\ref{thm:global}, which guarantees that all players identify the same $N$ cells (Lemma~\ref{lem:consensus} in Appendix~C).
Musical Chairs (Appendix~D) is then run on this fixed set; its randomness is independent of reward noise used to define $\mathcal E$.
Therefore,
\begin{equation}
\label{eq:mc-condition}
\EE\!\left[R_{MC}(T_{MC})\,\big|\,\mathcal E\right]
\ =\ \EE\!\left[R_{MC}(T_{MC})\right]
\ \le\ c_{MC}\,N^2,
\end{equation}
for an absolute constant $c_{MC}$ (Appendix~D).

\subsection*{F.3.\; In-cell term and end-to-end bound}

On $\mathcal E$, the Phase~III processes are collision-free and decoupled across players.
For any player $j$, Proposition~\ref{prop:in-cell} gives
\[
\EE\!\left[R_{\mathrm{in}}^{(C_j)}(T^{(j)}_{\mathrm{in}})\,\big|\,\mathcal E\right]
\ \le\ c_d\,(Lh)^{\frac{d}{d+2}}\,(T^{(j)}_{\mathrm{in}})^{\frac{d+1}{d+2}} + c'_d
\ \le\ c_d\,(Lh)^{\frac{d}{d+2}}\,T^{\frac{d+1}{d+2}} + c'_d,
\]
and summing over $j\in[N]$,
\begin{equation}
\label{eq:phaseIII-sum}
\sum_{j=1}^N \EE\!\left[R_{\mathrm{in}}^{(C_j)}(T^{(j)}_{\mathrm{in}})\,\big|\,\mathcal E\right]
\ \le\ c_d\,N\,(Lh)^{\frac{d}{d+2}}\,T^{\frac{d+1}{d+2}} + N c'_d.
\end{equation}
Combining \eqref{eq:failure-expectation}, \eqref{eq:phaseI-II}, \eqref{eq:mc-condition}, and \eqref{eq:phaseIII-sum} yields
\[
\EE[R_{\cont}(T)]
\ \le\ N(T_0+T_1) + c_{MC}N^2 + c_d\,N\,(Lh)^{\frac{d}{d+2}}\,T^{\frac{d+1}{d+2}} + N c'_d + \delta_{sys} T.
\]
Absorbing the additive $Nc'_d$ into the (horizon-independent) coordination constant completes the proof of Theorem~\ref{thm:global}.
\qed

\medskip

\noindent\textbf{Remarks.}
(i) The term $N(T_0+T_1)+c_{MC}N^2$ is strictly horizon-independent only when $\delta_{sys}$ is fixed. If one instead sets $\delta_{sys}=\delta_{sys}(T)$, e.g.\ $\delta_{sys}=1/T$ or $\delta_{sys}=1/(NT)$, so that the explicit failure contribution is $O(1)$ in expectation, then the dependence on $T$ enters only through the logarithms in the Phase~I/II radii and therefore changes the bound only by additional polylogarithmic factors.
(i) The quantity \(N(T_0+T_1)+c_{MC}N^2\) is strictly horizon-independent only when \(\delta_{sys}\) is treated as a fixed confidence parameter. If one instead chooses \(\delta_{sys}=\delta_{sys}(T)\) (for example \(\delta_{sys}=O(1/T)\)) so that the failure contribution \(\delta_{sys}T\) is \(O(1)\) in expectation, then the Phase~I/II radii acquire additional \(\log T\) factors, and so do \(T_0\) and \(T_1\). Thus, in the expected-regret view the coordination term is best understood as polylogarithmic in \(T\), rather than strictly \(T\)-independent. These logarithmic factors are absorbed into \(\tilde O(\cdot)\).
(ii) The proof above uses only consensus (from $\varepsilon$-uniqueness) to run MC on a fixed target set; no other structural gap is used in Phase~III.

\section*{Appendix G: Gap-Free Analysis—Consensus, Baselines, Limits, and a Conditional Epochic Recovery}
\label{app:gapfree}

This section gives a complete treatment of the gap-free regime referenced in Section~\ref{sec:global_guarantees}. Our goals are:
\begin{enumerate}
\item to formalize a communication-free public dither mechanism that guarantees consensus among players in Phase~II selection and to prove it correct with safe constants;
\item to provide fully general guarantees that hold without extra assumptions on the instance: a single-shot gap-free bound and a restart lower bound showing that restarting global identification each epoch is too costly under the Phase~II sampling analyzed in Appendices~B-D;
\item to state and prove a conditional epochic recovery theorem under an extra assumption compared to Phase~II (a public coverage/scheduling property).
\end{enumerate}

Throughout this appendix, we use the failure budgeting of Appendix~F so that the contribution of failure events to expected regret is at most $\delta_{\text{sys}}T$.

\subsection*{G.1.\; Communication-free consensus via a public dither with a guaranteed gap}
\label{subsec:public-dither}

In the gap-free regime, the refined LCBs $\{\LCB^{(1)}_j(C)\}_C$ produced in Phase~II can differ slightly across players, potentially leading to different top-$N$ sets. We enforce consensus without communication by adding a public deterministic \emph{dither} (randomness) $\xi(C)$ to the LCBs before ranking cells. Crucially, we ensure a minimum pairwise gap in $\xi$ so that the dither dominates cross-player LCB fluctuations for every pair of cells.

\paragraph{Internal vs.\ external precision and probe budget.}
Fix a target external precision $\varepsilon_{\text{main}}>0$ for the selection of $N$ cells at the end of Phase~II. Internally, Phase~II refines brackets to width
\[
\varepsilon_{\mathrm{int}}:=\frac{\varepsilon_{\text{main}}}{4},
\]
i.e., for all cells $C$ and all players $j$,
\[
\UCB^{(1)}_j(C)-\LCB^{(1)}_j(C)\ \le\ \varepsilon_{\mathrm{int}}.
\]
This is achieved in Appendix~C by choosing an $\eta$-net (probe spacing $\eta$) and a per-probe success budget $b$ so that $2r_1+L\eta\le \varepsilon_{\mathrm{int}}$, where $r_1=\sqrt{\beta_1/(2b)}$ and $\beta_1$ is the usual anytime log factor.\footnote{See Appendix~C (coverage lemma and anytime concentration), where we use law-of-total-probability removal of conditioning and an anytime union bound across the random counts $s\ge b$ to obtain uniform-in-time concentration at each probe.}

\paragraph{Public dither with a guaranteed minimum gap.}
Let $\{C_1,\dots,C_K\}$ be the cells in a fixed public order (e.g., lexicographic). Set
\[
\eta_{\mathrm{dit}}:=\frac{3\,\varepsilon_{\text{main}}}{4},
\qquad
\xi(C_m):=\frac{m-1}{K-1}\,\eta_{\mathrm{dit}},\quad m=1,\dots,K.
\]
Thus the minimum pairwise dither gap is $\Delta_\xi:=\eta_{\mathrm{dit}}/(K-1)$. 

We increase the per-probe success budget $b$ by a constant (in $T$) factor so that
\begin{equation}
\label{eq:dither-gap-dominates}
4r_1\ \le\ \Delta_\xi\ =\ \frac{\eta_{\mathrm{dit}}}{K-1}.
\end{equation}
(Equivalently, we reduce $r_1$ by a constant factor; this changes Phase~II constants but not rates in~$T$.) Each player ranks cells by the public \emph{score}
\[
\mathrm{Score}_j(C):=\LCB^{(1)}_j(C)+\xi(C)
\]
and selects the $N$ cells with largest Scores (ties broken lexicographically).

\begin{lemma}[Consensus and $\varepsilon_{\text{main}}$-optimality with public dither]
\label{lem:dither-consensus}
On the Phase~II accuracy event (Appendix~C) with bracket width $\le \varepsilon_{\mathrm{int}}$ for all cells, the public dither rule above ensures, with the same high probability:
\begin{enumerate}
\item \textbf{Consensus:} All players select the same top-$N$ set $S^{\mathrm{dit}}$.
\item \textbf{$\varepsilon_{\text{main}}$-optimality:} Every $C\in S^{\mathrm{dit}}$ satisfies $\mu^\ast(C)\ge \mu^\ast_{(N)}-\varepsilon_{\text{main}}$.
\end{enumerate}
\end{lemma}

\begin{proof}
(\emph{Consensus.}) On the Phase~II accuracy event, for any cell $C$ and players $j,k$,
\(
|\LCB^{(1)}_j(C)-\LCB^{(1)}_k(C)|\le 2r_1.
\)
Hence for any pair $(C,C')$,
\[
\sup_{j,k}\Big|\big(\LCB^{(1)}_j(C)-\LCB^{(1)}_j(C')\big)
-\big(\LCB^{(1)}_k(C)-\LCB^{(1)}_k(C')\big)\Big|
\le 4r_1.
\]
By construction $|\xi(C)-\xi(C')|\ge \Delta_\xi\ge 4r_1$ for  {every} pair $(C,C')$ (eq.~\eqref{eq:dither-gap-dominates}). Therefore, the sign of
\[
\big(\LCB^{(1)}_j(C)-\LCB^{(1)}_j(C')\big)+\big(\xi(C)-\xi(C')\big)
\]
is the same for all $j$, i.e., all players induce the same total order by $\mathrm{Score}_j(\cdot)$ and select the same top-$N$ set.

( {$\varepsilon_{\text{main}}$-optimality.}) Let $\theta^{\mathrm{dit}}$ be the $N$-th largest Score. For any true top-$N$ cell $C^\ast$,
$\LCB^{(1)}_j(C^\ast)\ge \mu^\ast(C^\ast)-\varepsilon_{\mathrm{int}}\ge \mu^\ast_{(N)}-\varepsilon_{\mathrm{int}}$ and $\xi(C^\ast)\ge 0$; hence $\theta^{\mathrm{dit}}\ge \mu^\ast_{(N)}-\varepsilon_{\mathrm{int}}$. For any selected $C$,
\[
\mu^\ast(C)\ \ge\ \LCB^{(1)}_j(C)\ \ge\ \theta^{\mathrm{dit}}-\xi(C)\ \ge\ \mu^\ast_{(N)}-(\varepsilon_{\mathrm{int}}+\eta_{\mathrm{dit}})\ =\ \mu^\ast_{(N)}-\varepsilon_{\text{main}}.\qedhere
\]
\end{proof}

\paragraph{Data reuse across epochs.}
Unless stated otherwise, we assume Phase~II reuses all probe data across epochs; we do not restart identification. Union over $K_{\text{ep}}=\Theta(\log T)$ epochs adds an extra $\log K_{\text{ep}}$ into the radii’s $\beta$, which is absorbed by $\tilde O(\cdot)$.

\subsection*{G.2.\; A fully general, single-shot gap-free guarantee}
\label{subsec:single-shot}

We first prove a gap-free guarantee that requires  {no} structural assumptions beyond Lipschitzness. Phase~I and Phase~II are run once at precision $\varepsilon$; then we apply Lemma~\ref{lem:dither-consensus} to select a common $N$-cell set, seat via Musical Chairs (Appendix~D), and run Phase~III (Appendix~E).

\begin{proposition}[Single-shot, gap-free baseline]
\label{prop:gap-free-single-shot}
For any $L$-Lipschitz mean on $[0,1]^d$ and horizon $T$, there exists a choice of $\varepsilon$ such that
\[
\EE[R_{\cont}(T)]\ =\ \tilde O\!\Big(N\, (K(Lh)^d)^{\frac{1}{d+3}}\, T^{\frac{d+2}{d+3}}\Big)\ =\ \tilde O\!\Big(N\, L^{\frac{d}{d+3}}\, T^{\frac{d+2}{d+3}}\Big),
\]
using $K=\lceil 1/h\rceil^d$ and $K h^d\in[1,2^d)$, whence $K(Lh)^d=\Theta(L^d)$. The $\tilde O(\cdot)$ hides logarithmic factors in $N,K,1/\delta_{\text{sys}}$.
\end{proposition}

\begin{proof}
On the clean event (Appendix~F), the Phase~II time at target width $\varepsilon$ satisfies (Appendix~C)
\[
T_1(\varepsilon)\ =\ \tilde O\!\big(M_{\mathrm{act}}(Lh)^d\,\varepsilon^{-(d+2)}\big),
\]
and in the worst case $M_{\mathrm{act}}\le K$. Each identification round contributes at most $N$ regret, hence
$R_{\mathrm{ID}}=\tilde O\!\big(NK(Lh)^d\,\varepsilon^{-(d+2)}\big)$. 

By Lemma~\ref{lem:dither-consensus}, Phase~III suboptimality is $R_{\mathrm{Sub}}\le N\varepsilon T$, while in-cell learning is $R_{\mathrm{Learn}}=\tilde O\!\big(N (Lh)^{\frac{d}{d+2}} T^{\frac{d+1}{d+2}}\big)$ (Appendix~E). 

Balancing $R_{\mathrm{ID}}$ and $R_{\mathrm{Sub}}$ gives $\varepsilon^\star\asymp (K(Lh)^d/T)^{1/(d+3)}$ and
\[
R_{\mathrm{ID}}+R_{\mathrm{Sub}}\ =\ \tilde O\!\Big(N\, (K(Lh)^d)^{\frac{1}{d+3}}\, T^{\frac{d+2}{d+3}}\Big).
\]
Since $\frac{d+2}{d+3}>\frac{d+1}{d+2}$ for $d\ge 1$, this dominates $R_{\mathrm{Learn}}$; adding the $O(N^2)$ seating constant and the $\delta_{\text{sys}}T$ failure contribution yields the claim.
\end{proof}

\subsection*{G.3.\; Why restart-style epochic identification fails}
\label{subsec:restart-fail}

We formalize that restarting Phase~II at each epoch (instead of reusing data) incurs linear identification overhead in worst-case Lipschitz instances.

\begin{proposition}[Lower bound for restart-style epochic identification]
\label{prop:epochic-lb}
Suppose that at the start of each epoch $k$ (of length $T_k$) the algorithm recomputes Phase~II brackets to width $\varepsilon_k$ by running the collision-censored sampling of Phases~I/II afresh, without reusing earlier probe data. Then for worst-case $L$-Lipschitz instances and any epoch schedule with $\sum_k T_k=T$,
\[
\EE\Big[\sum_k R_{\mathrm{ID}}^{(k)}\Big]\ =\ \Omega(NT).
\]
\end{proposition}

\begin{proof}
By Appendix~C (coverage and probe-level anytime bounds), achieving width $\varepsilon_k$ requires
\[
\Omega\!\big((Lh)^d\,\varepsilon_k^{-(d+2)}\big)
\]
rounds (up to logs), since the per-round success probability scales as $q_{M_{\mathrm{act}},\eta}\asymp 1/(M_{\mathrm{act}}P(\eta))$ with $M_{\mathrm{act}}\asymp K$ and $P(\eta)\asymp (Lh/\varepsilon_k)^d$ in the worst case. 

Each identification round contributes at most $N$ regret, so $R_{\mathrm{ID}}^{(k)}=\Omega\!\big(N(Lh)^d\,\varepsilon_k^{-(d+2)}\big)$. 

Balancing $N\varepsilon_k T_k$ and $N(Lh)^d\,\varepsilon_k^{-(d+2)}$ yields $\varepsilon_k\asymp T_k^{-1/(d+2)}$ and $\varepsilon_k^{-(d+2)}\asymp T_k$, hence $\sum_k R_{\mathrm{ID}}^{(k)}=\Omega\!\big(N(Lh)^d \sum_k T_k\big)=\Omega(NT)$ (absorbing $(Lh)^d$ into the constant if $h$ is fixed).
\end{proof}

\paragraph{Remark (data reuse).}
Proposition~\ref{prop:epochic-lb} targets  {restart} epochic identification. If Phase~II aggregates probe data across epochs (our default), the incremental work per epoch is smaller; however, without further structure (next section) the cumulative identification overhead remains too large to be dominated by the Phase~III learning term in the worst case.

\subsection*{G.4.\; Conditional epochic recovery under a public coverage/scheduling property}
\label{subsec:conditional-recovery}

We now state a proof of the rate in the main text. Given the negative result in the previous subsection, we require an additional condition (public coverage/scheduling) for Phase~II that is orthogonal to the reward statistics and can be viewed as a systems-level assumption.
We want to emphasize that this was not assumed in Appendices~B-D (which used uniform randomization over active cells and probes with collision censorship).

\begin{assumption}[Public stratified coverage for Phase~II]
\label{ass:coverage}
In any epoch with probe spacing $\eta$ and per-cell probe count $P(\eta)$, there exists a public, communication-free schedule with the following properties:
\begin{enumerate}[label=(\alph*)]
\item For each active probe $(C,z)$, there is a publicly known set $\mathcal{R}(C,z)$ of rounds with $|\mathcal{R}(C,z)|=\Theta(P(\eta))$ per $P(\eta)$-length block, and in each $t\in\mathcal{R}(C,z)$ exactly one player samples $z$ in $C$ and no other player samples any point in $C$; i.e., the per-round success probability for $(C,z)$ is $\Omega(1)$ on its scheduled rounds.
\item Across the epoch, the schedule assigns $O(1)$ such $(C,z)$ per round per player (constant load), enabling parallelism without collisions.
\end{enumerate}
\end{assumption}

\paragraph{Near-optimality (zooming) dimension.}
We use the standard instance-complexity notion: there exist $d^{*}\in[0,d]$ and $C>0$ such that $\X_\varepsilon:=\{x:\mu^{*}-\mu(x)\le \varepsilon\}$ can be covered by at most $C\,\varepsilon^{-d^{*}}$ Euclidean balls of radius $\Theta(\varepsilon/L)$ (see Kleinberg et al. \citep{Kleinberg}).

\begin{lemma}[Active-cell count under $d^{*}$]
\label{lem:Mact-dstar}
If $\X_\varepsilon$ admits a cover by $O(\varepsilon^{-d^{*}})$ balls of radius $c\,\varepsilon/L$ with $c\,\varepsilon/L\le h/4$, then the number of partition cells that intersect $\X_\varepsilon$ satisfies $M_{\mathrm{act}}(\varepsilon)=O(\varepsilon^{-d^{*}})$. For the finitely many coarser $\varepsilon$, the bound holds up to a constant factor absorbed in $\tilde O(\cdot)$.
\end{lemma}

\begin{proof}
Each ball of radius $c\,\varepsilon/L\le h/4$ intersects $O(1)$ cells; there are $O(\varepsilon^{-d^{*}})$ balls.
\end{proof}

\begin{proposition}[Per-epoch identification under Assumption~\ref{ass:coverage} and $d^{*}$]
\label{prop:per-epoch-coverage}
Fix an epoch with target width $\varepsilon_k$. Under Assumption~\ref{ass:coverage} and near-optimality dimension $d^{*}$, the number of successful probe  {samples} is $\tilde O(\varepsilon_k^{-(d^{*}+2)})$; since the per-probe success probability on scheduled rounds is $\Omega(1)$, the number of  {rounds} is also $\tilde O(\varepsilon_k^{-(d^{*}+2)})$ (up to logarithmic factors).
\end{proposition}

\begin{proof}
By Lemma~\ref{lem:Mact-dstar}, the number of probes is $O(\varepsilon_k^{-d^{*}})$. 

Each probe needs $b=\Theta(\varepsilon_k^{-2})$ successful samples to reach noise radius $O(\varepsilon_k)$ by the anytime concentration used in Appendix~C. 

Assumption~\ref{ass:coverage}(a)--(b) guarantees $\Omega(1)$ success probability for each probe on its scheduled rounds and constant load per round, hence the round budget is $\tilde O(\varepsilon_k^{-d^{*}}\cdot \varepsilon_k^{-2})$.
\end{proof}

\begin{theorem}[Conditional epochic, gap-free recovery]
\label{thm:conditional-epochic}
Assume $d^{*}\le d-1$ and Assumption~\ref{ass:coverage}. Run epochs of lengths $T_k=2^k$ with precisions $\varepsilon_k\propto 2^{-k/(d+2)}$, reuse all data across epochs, and use public dither (Lemma~\ref{lem:dither-consensus}). Then
\[
\sum_{k} \EE\big[R_{\mathrm{ID}}^{(k)}\big]\ =\ \tilde O\!\Big(N\,T^{\frac{d^{*}+2}{d+2}}\Big),
\]
which is dominated by the Phase~III learning term $\tilde O\!\big(N(Lh)^{\frac{d}{d+2}}T^{\frac{d+1}{d+2}}\big)$ when $d^{*}\le d-1$. Consequently, the main-text epochic gap-free corollary holds under Assumption~\ref{ass:coverage}.
\end{theorem}

\begin{proof}
By Proposition~\ref{prop:per-epoch-coverage}, per-epoch identification rounds are $\tilde O(\varepsilon_k^{-(d^{*}+2)})$, hence the identification regret is $N$ times this. 

With $\varepsilon_k\propto 2^{-k/(d+2)}$, summing a geometric series over $K_{\text{ep}}=\Theta(\log T)$ epochs yields $\tilde O\!\big(N\,T^{\frac{d^{*}+2}{d+2}}\big)$. 

Since $d^{*}\le d-1$, we have $\frac{d^{*}+2}{d+2}\le \frac{d+1}{d+2}$, so identification is dominated by Phase~III learning (Appendix~E). 

Add the $O(N^2)$ seating constant and the $\delta_{\text{sys}}T$ failure contribution (Appendix~F). 

Reusing probe data across epochs adds at most a $\log K_{\text{ep}}$ factor in the radii, absorbed by~$\tilde O(\cdot)$.
\end{proof}

\bigskip

\section*{Appendix H: Distance-Threshold Collisions via Packing Reduction}
\label{app:packing}

This appendix provides a reduction from the distance-threshold collision model to the partition-style analysis used in the main text. The regret comparator in Theorem~\ref{thm:packing} is the packing-based benchmark $\OPT_{\pack}(r,\sigma,N)$; without additional geometric covering assumptions, we do not claim a uniform comparison to the best $\rho$-separated assignment.

\subsection*{H.1.\; Setup and notation}

Fix a collision threshold $\rho>0$. Let $Z=\{z_1,\dots,z_M\}\subset \X$ be an  {$r$-packing}, i.e., $\|z_i-z_{i'}\|_2\ge r$ for all $i\ne i'$, with some $r>\rho$. We assume
\begin{equation}
\label{eq:packing-feasible}
M=|Z|\ \ge\ N,
\end{equation}
which is necessary both for the comparator $\OPT_{\pack}(r,\sigma,N)=\sum_{m=1}^N \nu^\ast_{(m)}$ and for seating $N$ players on distinct balls.

Fix a radius $\sigma$ with
\begin{equation}
\label{eq:sigma-safe}
0<\sigma<\frac{r-\rho}{2}.
\end{equation}
Define the  {safe balls} $B_i:=\{x\in\X:\|x-z_i\|_2\le\sigma\}$ (balls are implicitly clipped to $\X$). For a Lipschitz mean $\mu$, let $\nu_i^\ast:=\sup_{x\in B_i}\mu(x)$ and let $\nu_{(1)}^\ast\ge\cdots\ge \nu_{(M)}^\ast$ be the sorted values. The packing-based one-round benchmark is
\[
\OPT_{\pack}(r,\sigma,N)\ :=\ \sum_{m=1}^N \nu_{(m)}^\ast.
\]

\begin{lemma}[Collision safety across safe balls]
\label{lem:packing-safety}
If $\sigma$ satisfies \eqref{eq:sigma-safe}, then for any $i\ne i'$ and any $x\in B_i$, $x'\in B_{i'}$, we have $\|x-x'\|_2>\rho$. Hence inter-ball collisions are impossible.
\end{lemma}

\begin{proof}
Triangle inequality and $r$-packing: $\|x-x'\|_2 \ge \|z_i-z_{i'}\|_2-\|x-z_i\|_2-\|x'-z_{i'}\|_2 \ge r-2\sigma>\rho$.
\end{proof}

\subsection*{H.2.\; Reduction to the partition model}

We treat the index family $\{B_1,\dots,B_M\}$ as a ``virtual partition.” The multi-phase protocol (Phases I–III and II$\tfrac{1}{2}$) is run  {verbatim} with the following substitutions:
\begin{itemize}
\item \textbf{Phase I (coarse identification).} Each round, a player samples a ball index $I\in[M]$ uniformly and probes its center $z_I$ (or any fixed representative in $B_I$). The per-round success probability for a given $(\text{player},\text{ball})$ is
\[
p_M\ :=\ \frac{1}{M}\Bigl(1-\frac{1}{M}\Bigr)^{N-1},
\]
identical to the cell-based $p_K$ with $K$ replaced by $M$. The coarse brackets mirror Appendix~B with $K\mapsto M$ and with the ball geometry:
\[
\LCB^{(0)}_j(B_i)\ =\ \widehat\mu^{(0)}_j(z_i)-r^{(0)}_j(B_i),
\qquad
\UCB^{(0)}_j(B_i)\ =\ \widehat\mu^{(0)}_j(z_i)+r^{(0)}_j(B_i)+L\,\sigma,
\]
since $\mu^\ast(B_i)\le \mu(z_i)+L\max_{x\in B_i}\|x-z_i\|_2=\mu(z_i)+L\sigma$.
\item \textbf{Phase II (local peek).} For each active ball $B_i$, instantiate an internal $\eta$-net $Z_\eta(B_i)$; standard volumetric arguments give $|Z_\eta(B_i)|\le C_d\,(\sigma/\eta)^d$. All Phase‑II lemmas (coverage, anytime accuracy, refined brackets) from Appendix~C carry through with the replacements $K\mapsto M$, $h\mapsto \sigma$, $D_h\mapsto D_\sigma:=2\sigma$.
\item \textbf{Phase II$\tfrac{1}{2}$ (seating).} Musical Chairs is unchanged. By Lemma~\ref{lem:packing-safety}, once one player is seated per ball, inter-ball collisions cannot occur.
\item \textbf{Phase III (within-ball optimization).} With players uniquely assigned to balls, the processes decouple. Rescaling a ball of radius $\sigma$ to the unit ball multiplies the Lipschitz constant by $\sigma$; by Appendix~E, the in-ball regret over $T'$ rounds satisfies
\[
\EE\!\left[R_{\mathrm{in}}^{(B_i)}(T')\right]\ \le\ c_d\,(L\sigma)^{\frac{d}{d+2}}\,(T')^{\frac{d+1}{d+2}}+c'_d.
\]
\end{itemize}

\subsection*{H.3.\; Regret bound (proof of Theorem~\ref{thm:packing})}

Combining Phase~I/II identification costs, the seating cost (Appendix~D), and the sum of in-ball regrets (Appendix~E) exactly as in Appendix~F yields
\[
\EE[R_{\cont}(T)]
\ \le\
\underbrace{N(T_0+T_1)}_{\text{identification}}
\ +\ 
\underbrace{c_{MC}N^2}_{\text{seating}}
\ +\
\underbrace{c_d\,N\,(L\sigma)^{\frac{d}{d+2}}\,T^{\frac{d+1}{d+2}}}_{\text{learning}}
\ +\ \delta_{sys}T,
\]
where $T_1=\tilde O\!\big(M\,(L\sigma)^d\,\varepsilon^{-(d+2)}\big)$ (Appendix~C.6 with $K\mapsto M$ and $h\mapsto \sigma$). This is identical in form to Theorem~\ref{thm:global} after the substitutions $(K,h)\mapsto (M,\sigma)$.

\begin{remark}[Comparator choice]
\label{rem:pack-vs-true}
We compare to $\OPT_{\pack}(r,\sigma,N)$ by design. Uniformly relating the true $\rho$-separated optimum to $\OPT_{\pack}(r,\sigma,N)$ would require a  {covering} guarantee in addition to packing; without it, the two comparators need not be close. Even with coverage radius $O(r)$, the model-mismatch term per round can be $\Omega(NL r)$, which is not negligible for fixed $r$.
\end{remark}

\section*{Appendix I: Minimax Lower Bound}
\label{app:lower}

We prove Theorem~\ref{thm:lower}. The ingredients are: (i) a standard $\Omega(\sqrt{K\tau})$ minimax lower bound for  {finite-armed} stochastic bandits, and (ii) an $L$-Lipschitz embedding of $K'=\Theta(m^d)$ ``arms” within each of $N$ distinct cells, with a cone-shaped spike at one grid point.

\subsection*{I.1.\; Finite-armed minimax lower bound (cited)}
For every $K\ge 2$ and $\tau\ge K$, there exists a universal constant $c>0$ such that
\begin{equation}
\label{eq:finite-lb}
\inf_{Alg}\ \sup_{\nu}\ \EE_{\nu}\!\left[\sum_{t=1}^{\tau} (\mu^\ast-\mu(A_t))\right]\ \ge\ c\,\sqrt{K\,\tau},
\end{equation}
where the supremum is over all product reward distributions with means in $[0,1]$. See, e.g., Slivkins \citep{slivkins2019introduction}.

\subsection*{I.2.\; Proof of Theorem~\ref{thm:lower}}

Partition $[0,1]^d$ into $K=\lceil 1/h\rceil^d$ hypercubes $\Pcal=\{C_1,\dots,C_K\}$ of side $h$. Assume
\[
K=\lceil 1/h\rceil^d\ \ge\ N,
\]
and select $N$ distinct cells $\{C_{i_1},\dots,C_{i_N}\}$. Within each chosen cell $C_{i_j}$, place a regular grid $G$ of resolution $m$ per coordinate; neighboring grid points are at Euclidean distance $h/m$. Let $K'=(m+1)^d$ be the number of grid points in $G$.

For $\theta=(\theta_1,\dots,\theta_N)$ with $\theta_j\in\{1,\dots,K'\}$, define a Lipschitz mean $\mu_\theta$ as follows. In each chosen cell $C_{i_j}$, let $g_{j,\theta_j}\in G$ be the selected “spike” grid point and set
\[
\mu_\theta(x)\ :=\
\max\!\Big\{\tfrac{1}{2},\ \tfrac{1}{2}+\Delta\ -\ L\,\|x-g_{j,\theta_j}\|_2\Big\},
\qquad x\in C_{i_j},
\]
and set $\mu_\theta(x)\equiv\tfrac{1}{2}$ for $x$ in all other cells. Since $x\mapsto \|x-g\|_2$ is 1-Lipschitz and $\max(\cdot,\cdot)$ preserves Lipschitz constant, we have $\mu_\theta\in\mathcal L(L)$. At the spike $g_{j,\theta_j}$, the value is $1/2+\Delta$; at any other grid point $g\neq g_{j,\theta_j}$, the value is at most $1/2+\Delta - L(h/m)$, hence the spike arm is uniquely optimal.

Let $\OPT_{\cont}(\Pcal,N)$ be the per-round partition benchmark. In our construction, exactly the $N$ chosen cells attain maximum $1/2+\Delta$ and all others have maximum $1/2$, so
\[
\OPT_{\cont}(\Pcal,N)\ =\ \sum_{j=1}^{N} \Big(\tfrac{1}{2}+\Delta\Big)\ =\ N\Big(\tfrac{1}{2}+\Delta\Big).
\]

Consider the product prior over $\theta$ in which each $\theta_j$ is independent and uniform over the $K'$ grid points in $C_{i_j}$. By Yao’s minimax principle, the minimax expected regret is bounded below by the Bayes expected regret under this prior. Rewards and choices decouple across the $N$ special cells, and the per-round benchmark is additive across cells; therefore the Bayes expected regret equals the sum of the $N$ per-cell Bayes regrets. Each per-cell problem is a $K'$-armed stochastic bandit with rewards in $[0,1]$ and horizon $T$, so by \eqref{eq:finite-lb} the per-cell Bayes (hence minimax) regret is $\Omega(\sqrt{K' T})$. Summing over cells,
\[
\EE_{\mu_\theta}\!\left[R_{\cont}(T)\right]\ \ge\ c\,N\,\sqrt{K' T}\ \asymp\ c\,N\,m^{d/2}\,\sqrt{T}.
\]

Finally, choose parameters to ensure bounded rewards (take $\Delta\le 1/6$) and to match the worst-case finite-armed scaling under the $L$-Lipschitz constraint. Set
\[
\Delta\ =\ c_0\,\frac{L h}{m}\quad\text{with}\quad c_0\in(0,1/2],
\]
so that the spike height is consistent with the grid spacing and the gap at the nearest other grid point is at least $\Delta/2$. Optimizing the lower bound in $m$ (subject to $m\ge 2$ and $\Delta\le 1/6$) yields
\[
m\ \asymp\ (L h)^{\frac{2}{d+2}}\,T^{\frac{1}{d+2}},
\]
and hence
\[
\EE_{\mu_\theta}\!\left[R_{\cont}(T)\right]\ \gtrsim\ N\,(L h)^{\frac{d}{d+2}}\,T^{\frac{d+1}{d+2}}.
\]
Taking the supremum over $\theta$ concludes the proof of Theorem~\ref{thm:lower}.
\qed

\medskip

\noindent\textbf{Remarks.}
(i) Disjointness of cells suffices; ``non-adjacent” is inessential since collisions are intra-cell only.  
(ii) The same construction applies to the distance-threshold model by embedding spikes in $N$ safe balls of a packing (Appendix~H); at most one non-colliding observation per ball per round is possible, and the $N$-fold lower bound follows identically.

\section*{Appendix J: Experiments}
\label{app:experiments}

In this section, we empirically validate our theory using simulation results. In particular, we focus on three questions that mirror the main analytical claims:
\begin{enumerate}
    \item Does the coordination-first protocol produce substantially smaller regret than a naive decentralized baseline?
    \item Are collisions in fact concentrated in the short coordination stage, rather than persisting throughout learning?
    \item Does the local-peek step matter in practice, or could one simply rank cells by their center values?
\end{enumerate}

\paragraph{Experimental setup.}
We work in the partition-based collision model from Section~\ref{sec:setup}. In each run, the mean reward is a Lipschitz function on $[0,1]^d$ obtained from a small number of cone-shaped peaks. Observed rewards are Bernoulli with mean $\mu(x)$. Regret is measured against exactly the same benchmark used throughout the paper, namely the sum of the top-$N$ cell maxima from \eqref{eq:cont-regret}. We consider one-dimensional and two-dimensional instances, since these are the smallest settings in which both the geometric structure and the collision effects are easy to visualize.

The full experimental configuration is listed in Table~\ref{tab:exp-config}. Plots are averaged over $5$ random seeds and the shaded bands in the plots denote $95\%$ confidence intervals across seeds.

\paragraph{Methods compared.}
Since our paper proposes a new setting, we do not have a clear baseline to compare to. Thus, we consider the naive method where each player runs an independent single-agent Lipschitz bandit routine over the \emph{entire} domain and ignores the presence of the other players except through the collision-censored feedback as a natural starting point for comparison. Concretely, each player uses the same fixed-grid UCB primitive that our method uses in its final within-cell stage, but applies it globally rather than after coordination. This makes the comparison easy to read as the principal difference is not the local learning rule, but the presence or absence of an explicit coordination stage.

For this empirical study, we pool the successful Phase-I and Phase-II identification samples when forming the common target set. We do this to keep the experiments focused on the coordination-versus-learning decomposition that is central to the paper, rather than on incidental finite-sample disagreement between players during identification. The seating stage and the Phase-III local optimization stage are otherwise unchanged.

\begin{table}[t]
\centering
\small
\begin{tabular}{lccccccc}
\hline
Setting & $T$ & $N$ & Cells & $T_0$ & $T_1$ & Local grid & Seeds \\
\hline
1D regret & $10{,}000$ & $3$ & $8$ & $260$ & $700$ & $7$ & $5$ \\
2D regret & $10{,}000$ & $3$ & $4\times 4$ & $520$ & $1100$ & $5\times 5$ & $5$ \\
Pathology illustration & -- & $2$ & $6$ & -- & $9$ probes/cell & -- & deterministic \\
\hline
\end{tabular}
\caption{\textbf{Synthetic environments and algorithmic settings.} ``Local grid'' denotes the number of candidate points per cell used by the within-cell UCB routine in Phase III.}
\label{tab:exp-config}
\end{table}

\paragraph{Regret curves.}
Figure~\ref{fig:appendix-empirical-regret} reports cumulative regret in 1D and 2D. The qualitative picture is the same in both dimensions. The independent baseline repeatedly directs multiple players toward the same attractive cells and therefore accumulates regret at an essentially linear rate. By contrast, our protocol pays a short upfront price to identify distinct high-value cells and seat the players on them. Once that coordination cost has been paid, the learning problem largely decouples across players, and the subsequent regret growth is much slower.

\begin{figure}[t]
    \centering
    \includegraphics[width=0.99\linewidth]{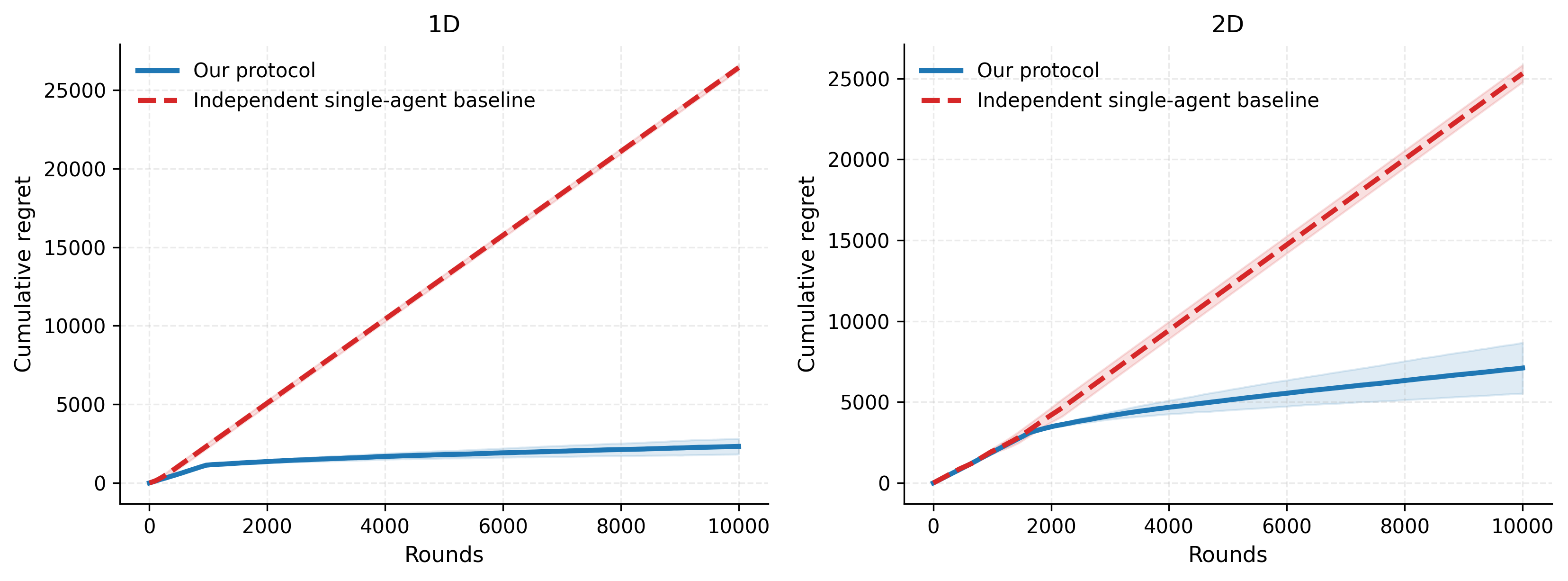}
    \caption{\textbf{Cumulative regret in simple synthetic instances.} We compare the proposed protocol against the independent single-agent baseline in 1D and 2D. Curves are averaged over $5$ random seeds and shaded bands denote $95\%$ confidence intervals. The baseline exhibits near-linear regret because players continue to collide on the same high-value cells, whereas the proposed protocol incurs a short coordination cost and then grows much more slowly.}
    \label{fig:appendix-empirical-regret}
\end{figure}

\begin{table}[t]
\centering
\small
\begin{tabular}{lccc}
\hline
Setting & Our protocol & Independent baseline & Mean seating rounds \\
\hline
1D at $T=10{,}000$ & $2326.2$ & $26432.4$ & $3.6$ \\
2D at $T=10{,}000$ & $7102.4$ & $25306.4$ & $2.2$ \\
\hline
\end{tabular}
\caption{\textbf{Final regret summary.} Values are mean cumulative regret at the final horizon in Figure~\ref{fig:appendix-empirical-regret}. The last column reports the average number of rounds spent in the seating stage by the proposed protocol.}
\label{tab:exp-summary}
\end{table}

\paragraph{Collision dynamics.}
The regret curves become easier to interpret once we inspect the collisions directly. Figure~\ref{fig:empirical-collisions} shows the smoothed fraction of colliding players over time, with vertical markers denoting the ends of Phase I and Phase II. The baseline keeps colliding throughout the run; every player is effectively trying to solve the same global problem, so there is no mechanism for persistent deconfliction. Our method behaves very differently. Collisions are concentrated in the short identification and seating stages, after which they nearly disappear once players occupy distinct cells. This is precisely the operational picture suggested by the theory where collisions are an upfront coordination cost.

\begin{figure}[t]
    \centering
    \includegraphics[width=0.99\linewidth]{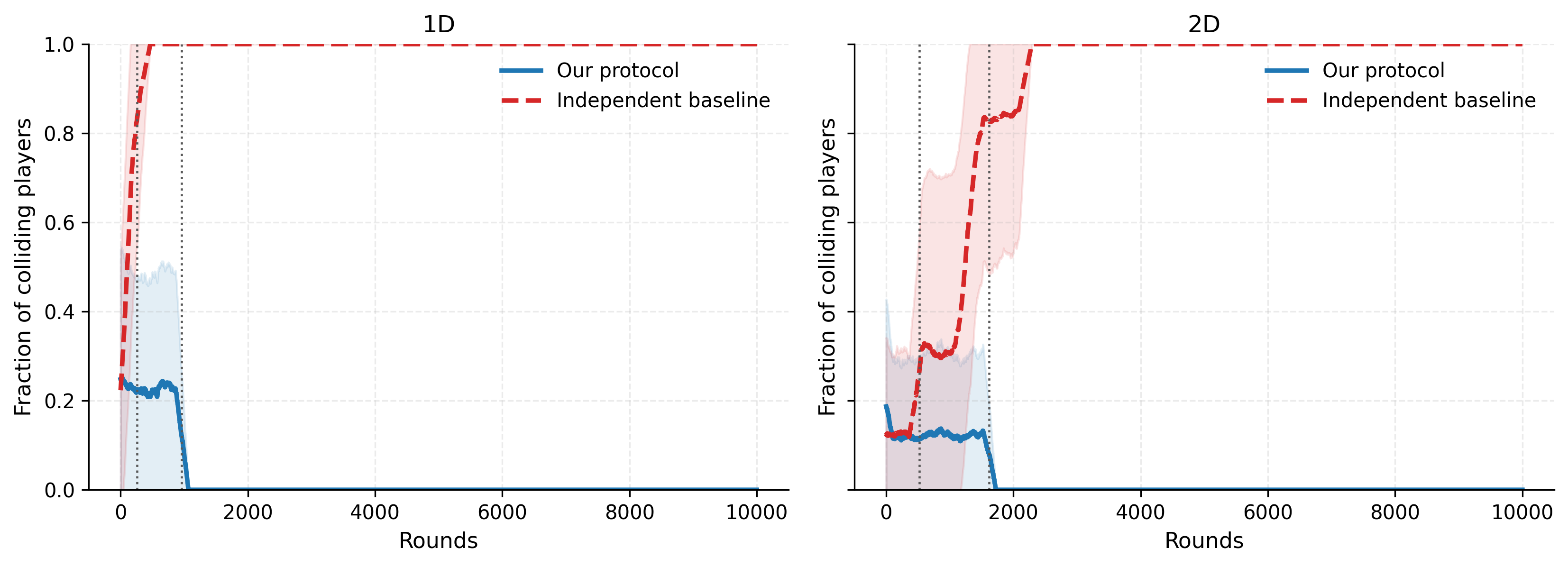}
    \caption{\textbf{Collision traces in 1D and 2D.} Dotted vertical lines mark the ends of Phase I and Phase II. The proposed protocol localizes collisions to the short upfront coordination stage, while the independent baseline continues to collide almost all the time.}
    \label{fig:empirical-collisions}
\end{figure}

\paragraph{Why the local peek matters.}
Finally, Figure~\ref{fig:empirical-pathology} illustrates the specific geometric issue that motivates Phase II. The dominant peak lies very close to the boundary between cells $C_3$ and $C_4$. As a result, those two cells have the largest within-cell maxima, even though their center values are not the largest. In this instance, center-based ranking prefers $C_5$ and $C_6$ because their centers happen to lie in moderately strong regions, while the true top-$N$ cells are actually $C_3$ and $C_4$. The local-peek scores correct this mis-ranking by probing within each cell and therefore recover the correct pair.

\begin{figure}[t]
    \centering
    \includegraphics[width=0.99\linewidth]{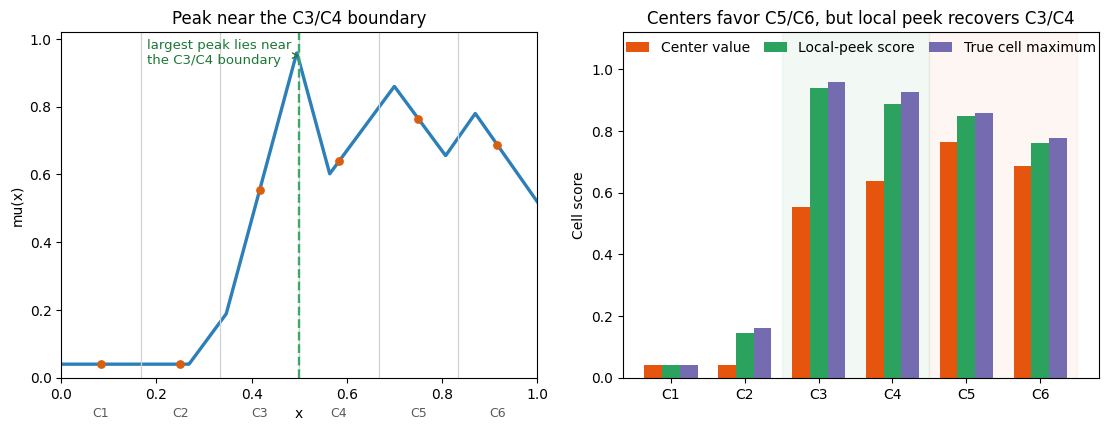}
    \caption{\textbf{Boundary-peak pathology.} Left: a one-dimensional reward function in which the largest peak lies near the boundary between $C_3$ and $C_4$; orange markers denote the cell centers. Right: orange bars are center values, green bars are local-peek scores, and purple bars are true cell maxima. The center values are largest for $C_5$ and $C_6$, so a center-based ranking would choose those cells. In contrast, the true cell maxima are largest for $C_3$ and $C_4$, and the local-peek scores recover exactly that ordering. This is the geometric failure mode that motivates Phase II.}
    \label{fig:empirical-pathology}
\end{figure}

Taken together, these experiments support the main qualitative message of the paper. The primary difficulty in this setting is the need to coordinate players onto different high-value regions without communication. Once that coordination is handled, the remaining learning problem behaves much more like a collection of ordinary single-player Lipschitz bandits.

All figures in this appendix can be regenerated from the repository at: https://github.com/amitrege/aistats\_multiagent\_code

\section*{Appendix K: Parameter Summary}
\label{app:param-summary}

Table~\ref{tab:param-summary} collects the main quantities used in Phases I and II and the sufficient choices proved in Appendices~B and~C.

\begin{table}[h]
\centering
\small
\renewcommand{\arraystretch}{1.15}
\begin{tabular}{|p{0.16\linewidth}|p{0.28\linewidth}|p{0.46\linewidth}|}
\hline
Quantity & Meaning & Sufficient choice / scaling used in the analysis \\
\hline
\(p_K\) & Phase-I per-round success probability for a fixed (player, cell) pair & \(p_K=\frac{1}{K}\left(1-\frac{1}{K}\right)^{N-1}\) \\
\hline
\(\alpha\) & Target Phase-I center-estimation radius & User-chosen coarse accuracy level \\
\hline
\(T_0\) & Phase-I budget & It suffices that \eqref{eq:T0-consolidated} holds; equivalently \(T_0=\tilde O(p_K^{-1}\alpha^{-2})\) for fixed \(\alpha\) and fixed failure budget \\
\hline
\(\eta\) & Phase-II probe-net resolution & For target final maxima accuracy \(\varepsilon\), choose \(\eta=\varepsilon/(2L)\) \\
\hline
\(P_{\max}\) & Maximum number of probe points in one active cell & \(P_{\max}\le C_d (h/\eta)^d\) \\
\hline
\(\NumProbes\) & Total number of active probe triples \((j,C,z)\) & \(\NumProbes \le N M_{act} P_{\max}\) \\
\hline
\(b\) & Required successful samples per active probe & Choose \(b \ge \max\!\left\{4\log\frac{2\NumProbes}{\delta_{II}}, \frac{8\beta_1}{\varepsilon^2}\right\}\), where \(\beta_1=\log\frac{4\NumProbes}{\delta_{II}}\) \\
\hline
\(q_{M_{act},\eta}\) & Phase-II per-round success probability lower bound for a probe triple & \(q_{M_{act},\eta} \gtrsim \frac{1}{M_{act}}\left(\frac{\eta}{h}\right)^d \left(1-\frac{1}{N}\right)^{N-1}\) \\
\hline
\(T_1\) & Phase-II budget & \(T_1 \ge 2b/q_{M_{act},\eta}\), hence \(T_1=\tilde O\!\big(M_{act}(Lh)^d\varepsilon^{-(d+2)}\big)\) \\
\hline
\(\delta_I,\delta_{II}\) & Phase-wise failure budgets & Chosen so that \(\delta_I+\delta_{II}\le \delta_{sys}\); if \(\delta_{sys}\) depends on \(T\), the resulting \(T_0,T_1\) gain only extra logarithmic factors \\
\hline
\end{tabular}
\caption{Summary of the main design parameters and sufficient choices. The displayed bounds are sufficient choices used in the proofs, not optimized minima.}
\label{tab:param-summary}
\end{table}

\end{document}